\documentclass{article}

\usepackage[preprint]{corl_2026}

\usepackage{amsmath,amssymb,mathtools,amsthm}
\usepackage{booktabs}
\usepackage{graphicx}
\usepackage{xcolor}
\usepackage{enumitem}
\usepackage{array}
\usepackage{multirow}
\usepackage{subcaption}

\newtheorem{proposition}{Proposition}
\newtheorem{lemma}{Lemma}
\newtheorem{remark}{Remark}

\newcommand{\cvloss}{\mathcal{L}_{\mathrm{CV}}}

\title{Cross-View Action Consistency for Camera-Robust Vision-Language-Action Policies}
\author{%
  Bingqi Huang \\
  Tsinghua University \\
  \texttt{hbq21@mails.tsinghua.edu.cn}
  \And
  Bingchuan Wei \\
  Tsinghua University
  \And
  Xuan Wang \\
  Informatics Institute, University of Amsterdam \\
  \& Faculty of Science, Vrije Universiteit Amsterdam
  \AND
  Yingkai Cai \\
  Tsinghua University
  \And
  Zhaokui Wang \\
  Tsinghua University
}

\begin{document}
\raggedbottom
\maketitle
\hypersetup{pdfauthor={Bingqi Huang, Bingchuan Wei, Xuan Wang, Yingkai Cai, Zhaokui Wang}, pdftitle={Cross-View Action Consistency for Camera-Robust Vision-Language-Action Policies}, pdfsubject={}}

\begin{abstract}
\label{sec:abstract}
Vision-language-action (VLA) policies fine-tuned from a fixed scene camera can fail when the camera is moved, even when the task, objects, language, and robot state are unchanged. We study scene-camera viewpoint robustness using only a scene RGB image, language, and proprioception, without camera labels, extrinsics, depth, or point-cloud inputs. The wrist stream is masked throughout to prevent an unperturbed visual shortcut from confounding attribution to scene-camera variation. For flow-based VLAs, we propose to regularize the action-flow velocity field, the quantity directly integrated to generate continuous action chunks. We construct action-equivalent view pairs by resetting original LIBERO demonstrations to the same MuJoCo state and rendering nominal and perturbed scene-camera views. Both views are supervised by flow matching, while a cross-view loss encourages their predicted action-flow velocities to agree at the same sampled flow coordinates. On the LIBERO-Plus camera-perturbation track, our method reaches $\mathbf{87.2{\pm}0.4\%}$ ($4{,}797$ rollouts per seed across $3$ training seeds), $+7.4$pp over flow-matching-only training on the same paired data ($79.8{\pm}0.8\%$, also $3$ seeds) and $+12.5$pp over naive mixed-camera SFT, and does not cost nominal-camera success ($94.2{\pm}0.8\%$ vs.\ $93.6{\pm}0.4\%$ for same-data FM-only). A shuffled-pair control collapses to $25.8\%$, showing that the gain depends on action-equivalent pairing. When the training pairs are restricted to part of the camera-pose range, the gain extends beyond the paired support on azimuth, camera distance, and endpoint rotation, most on azimuth. On a real robot, we evaluate three tabletop tasks with $10$ rollouts per task and camera placement; held-out-camera success improves from $53.3\%$ to $74.4\%$ under the same single-scene-RGB inference interface.
\end{abstract}


\begin{figure}[t]
  \centering
  \includegraphics[width=\textwidth]{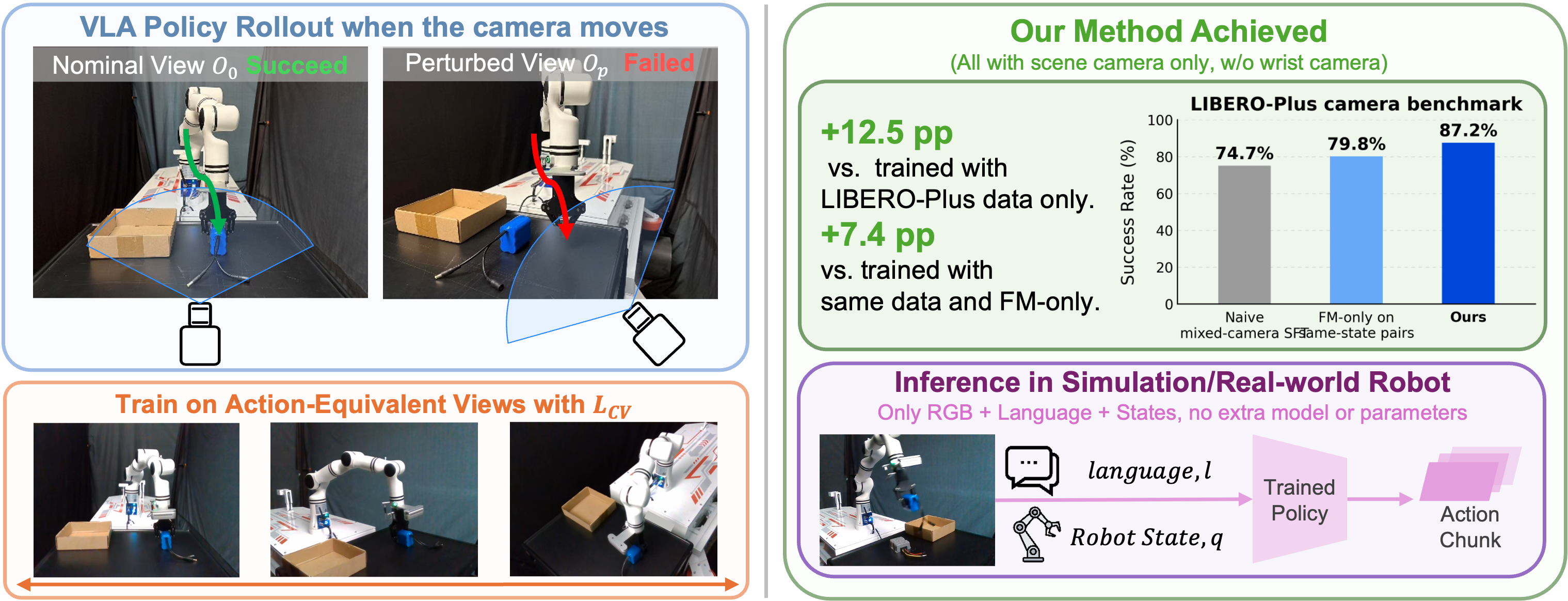}
  \caption{Camera viewpoint shifts can cause a VLA policy to produce different actions from the same physical state. We train on action-equivalent views so that different views of the same state are supervised by the same demonstrated action and regularized to produce consistent action-flow predictions. At inference, the policy receives a single RGB scene image (held-out camera placement), language instruction, and proprioceptive state; the wrist stream is masked throughout as an experimental control.}
  \label{fig:teaser}
\end{figure}

\section{Introduction}
\label{sec:intro}
Vision-language-action policies map RGB observations, language instructions, and robot state to robot actions~\citep{brohan2023rt1,brohan2024rt2,kim2024openvla,black2024pi0,intelligence2025pi05}. In practice, a pretrained VLA is often adapted from demonstrations collected with a fixed scene camera, but the same policy may degrade when the camera is bumped, re-mounted, or moved in distance, height, or orientation. Existing remedies often change the sensing or inference contract through depth, point clouds, calibrated multi-view inputs, camera rays, camera labels, novel-view rendering, or geometry-aware architectures~\citep{tian2024vista,qu2025spatialvla,abouzeid2025geoaware,li2025pointvla,li2025bridgevla,jiang2025cameracond,zhang2025ocvla,heo2026anycamvla,singh2025ogvla}. Camera-diverse supervised fine-tuning is simpler, but data diversity alone does not impose action-level invariance between equivalent views. We ask whether a flow-based VLA can become robust to scene-camera motion without depth, point clouds, camera labels, or geometric inputs.

The desired invariance is action-level. If different images observe the same physical state and share the same instruction and robot state, the policy should intend the same action. For flow-based VLAs, actions are generated by integrating a learned velocity field $v_\theta(x_t,t\mid o,l,q)$ over noisy action chunks. We therefore regularize the object that directly produces actions: the local action-flow velocity prediction at the same flow coordinate. This differs from view-invariant representation learning, where matching an intermediate feature does not necessarily imply matching the action distribution.

We implement this idea with action-equivalent multi-view pairs: two scene images of the same physical state (nominal and perturbed camera) sharing language, proprioception, and demonstrated actions. During training, both views receive flow matching to the same action target, and their predicted velocities are encouraged to agree at the same sampled flow coordinate; this cross-view loss is used only during training.

We evaluate on the LIBERO-Plus camera-perturbation track~\citep{liu2023libero,fei2025liberoplus} against a nominal-only baseline, camera-diverse supervised fine-tuning, and---most informatively---flow-matching-only training on the same paired data, the last of which isolates the cross-view term from pair-data exposure. The proposed objective reaches $87.2{\pm}0.4\%$ over three training seeds, $+7.4$pp over the same-data FM-only control, which also uses the same three seeds. A shuffled-pair control that breaks action equivalence inside the cross-view loss while preserving per-view flow-matching labels collapses performance, showing that the gain depends on action-equivalent correspondence rather than generic smoothing. Restricting the training pairs to part of the camera-pose range shows that the gain extends to held-out azimuth, camera-distance, and endpoint-rotation poses. A real-robot held-out-camera study with synchronized multi-camera demonstrations confirms that the same training principle transfers beyond simulator rerendering.

\textbf{Contributions.} (1) We formulate scene-camera viewpoint robustness for flow-based VLAs as local action-flow consistency across action-equivalent views, and propose a cross-view action-flow consistency objective that regularizes the velocity field for action generation directly, while adding no sensors, calibration, or camera-specific inputs at inference. (2) Controlled ablations show that matched cross-view action-flow consistency improves over camera-diverse flow matching, and that its benefit depends on action-equivalent pairing rather than generic regularization.

\section{Related Work}
\label{sec:related}

\textbf{Vision-language-action policies and flow action heads.}
VLAs combine visual encoders, language-conditioned transformers, and action decoders to control robots from RGB observations and instructions~\citep{brohan2023rt1,brohan2024rt2,kim2024openvla}. Recent systems generate continuous action chunks with diffusion or flow-style objectives~\citep{chi2023diffpol,black2024pi0,intelligence2025pi05,zhang2025flowpolicy}. We build on this formulation: the policy predicts a velocity field in action space, and our regularizer acts directly on that field.

\textbf{Camera and geometry-aware robot policies.}
Camera robustness can be introduced at different stages. Some methods change the deployment-time inputs or computation through spatial or geometry-aware VLA architectures~\citep{qu2025spatialvla,abouzeid2025geoaware}, point-cloud and 3D-aligned inputs~\citep{li2025pointvla,li2025bridgevla}, camera conditioning~\citep{jiang2025cameracond}, orthographic view generation~\citep{singh2025ogvla}, observation-centric action formulations~\citep{zhang2025ocvla}, test-time view restoration~\citep{heo2026anycamvla}, or learned representation recalibration~\citep{vlageneralizable2025}. Other approaches synthesize additional viewpoints during training while leaving the deployed policy interface unchanged~\citep{tian2024vista}. Our method also leaves deployment unchanged, but uses observed action-equivalent views to regularize the training objective rather than synthesizing a replacement view. Appendix~\ref{app:scope_comparison} contrasts these requirements across method families.

\textbf{Multi-view supervision and invariance.}
Multi-view data can teach representations or policies to ignore nuisance viewpoint changes~\citep{seo2023mvmwm,pang2025reviwo}. Generic robot visual pretraining also improves downstream manipulation~\citep{nair2022r3m,xiao2022mvp,majumdar2023vc1,karamcheti2023voltron}. However, feature invariance and policy invariance are not equivalent: a representation may match across views while the action decoder still changes, or it may preserve view-specific information that is useful but should not alter the action for an equivalent state. We therefore place consistency at the action-flow output of the policy rather than at a fixed hidden representation; feature-level alternatives we explored (Appendix~\ref{app:feature_alternatives}) did not yield a reliable camera-track gain.

\section{Cross-View Action-Flow Consistency}
\label{sec:method}

\begin{figure}[b]
  \centering
  \includegraphics[width=\textwidth]{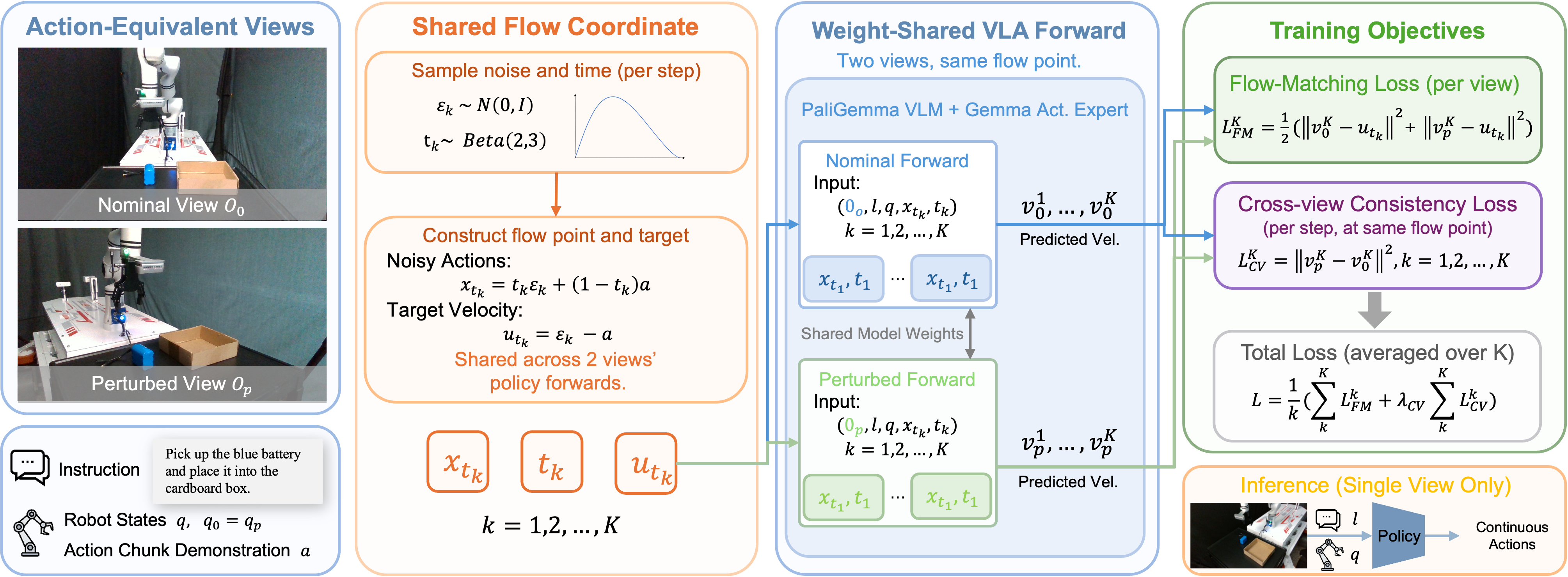}
  \caption{Cross-view action-flow consistency. During training, an action-equivalent nominal/perturbed view pair is processed by weight-shared VLA forward passes. Both views remain supervised by the standard flow-matching objective, while a cross-view loss encourages their action-flow predictions to agree at the same sampled flow coordinates. The pair is used only during training; inference uses a single RGB scene image, language instruction, and proprioceptive state.}
  \label{fig:method}
\end{figure}

\textbf{Action-equivalent pairs.}
Let $s$ be the physical task state, $c$ a scene-camera configuration, and $o=h(s,c)$ the rendered RGB observation. A training pair is
\begin{equation}
  (o_0,o_p,l,q,a), \qquad o_0=h(s,c_0),\quad o_p=h(s,c_p),
\end{equation}
where $c_0$ is the nominal camera, $c_p$ is a perturbed camera, $l$ is the language instruction, $q$ is the proprioceptive state, and $a\in\mathbb{R}^{H\times A_{\rm act}}$ is the demonstrated action chunk, where $H$ is the prediction horizon and $A_{\rm act}$ is the number of active action dimensions. The two images differ only by scene camera; the task state, instruction, proprioception, and action target are shared. 

\textbf{Flow-level consistency.}
We use the convention that $t=1$ is noise and $t=0$ is the demonstrated action. $x_t$ denotes the action sample at flow time $t$, and $u_t$ the corresponding conditional velocity target.  With $\varepsilon\in\mathbb{R}^{H\times A_{\mathrm{act}}}$ drawn from $\mathcal{N}(0,I)$,
\begin{equation}
  x_t = t\varepsilon + (1-t)a,
  \qquad
  u_t = \varepsilon-a.
  \label{eq:flow-coordinates}
\end{equation}
For each action-equivalent pair, both views share $a$, $x_t$, and $u_t$. For $K$ sampled flow-time/noise points $\{(t_k,\varepsilon_k)\}_{k=1}^{K}$, define
\begin{equation}
\begin{aligned}
  x_{t_k}&=t_k\varepsilon_k+(1-t_k)a,
  \qquad
  u_{t_k}=\varepsilon_k-a,\\
  v_0^k&=v_\theta(x_{t_k},t_k\mid o_0,l,q),
  \qquad
  v_p^k=v_\theta(x_{t_k},t_k\mid o_p,l,q).
\end{aligned}
\label{eq:paired-flow-predictions}
\end{equation}
Both views are anchored to the demonstrated action through standard flow matching:
\begin{equation}
  \mathcal{L}_{\mathrm{FM}}^k
  =
  \tfrac{1}{2}
  \left(
    \lVert v_0^k-u_{t_k}\rVert^2
    +
    \lVert v_p^k-u_{t_k}\rVert^2
  \right).
  \label{eq:paired-fm-loss}
\end{equation}
The cross-view term compares the two velocity predictions on active action dimensions:
\begin{equation}
  \mathcal{L}_{\mathrm{CV}}^k
  =
  \left\lVert
    v_p^k[\,\cdot,\,:A_{\mathrm{act}}\,]
    -
    v_0^k[\,\cdot,\,:A_{\mathrm{act}}\,]
  \right\rVert^2 .
  \label{eq:paired-cv-loss}
\end{equation}
The per-pair training objective is
\begin{equation}
  \mathcal{L}
  =
  \frac{1}{K}\sum_{k=1}^{K}\mathcal{L}_{\mathrm{FM}}^k
  +
  \lambda_{\mathrm{CV}}
  \frac{1}{K}\sum_{k=1}^{K}\mathcal{L}_{\mathrm{CV}}^k.
  \label{eq:main_loss}
\end{equation}
where $\lambda_{\mathrm{CV}} > 0$ is a scalar weight controlling the strength of the cross-view regularization. The same $(t_k,\varepsilon_k)$ is shared across both views within each pair, and gradients flow through both branches bilaterally. 

\textbf{Mean--residual reformulation.}
Writing the two-view predictions in mean--residual coordinates $\bar v_k=\tfrac{1}{2}(v_0^k+v_p^k)$ and $\delta_k=\tfrac{1}{2}(v_0^k-v_p^k)$, the per-pair objective in~\eqref{eq:main_loss} admits, on the coordinates both terms cover, the equivalent form
\begin{equation}
  \mathcal{L}^k
  =
  \lVert\bar v_k-u_{t_k}\rVert^2
  +
  (1+4c\lambda_{\mathrm{CV}})\,\lVert\delta_k\rVert^2,
  \label{eq:mean-residual}
\end{equation}
using
$\tfrac{1}{2}\!\left(\lVert v_0^k-u_{t_k}\rVert^2+\lVert v_p^k-u_{t_k}\rVert^2\right)=\lVert\bar v_k-u_{t_k}\rVert^2+\lVert\delta_k\rVert^2$
and $\lVert v_p^k-v_0^k\rVert^2=4\lVert\delta_k\rVert^2$, where $c$ is the ratio of the normalizations of the two terms and equals one when both average over the same coordinates. The demonstrated action enters only through the mean prediction $\bar v_k$, while the cross-view term acts purely on the view-disagreement residual $\delta_k$. Adding $\mathcal{L}_{\mathrm{CV}}$ therefore introduces no new target; it raises the penalty on view-specific velocity residuals beyond standard flow matching by a factor of $1+4c\lambda_{\mathrm{CV}}$, and leaves the weight at one on coordinates it does not cover. In our implementation the flow-matching loss averages over the padded action width of $32$ and the cross-view term over the $A_{\mathrm{act}}=7$ active dimensions, so $c=32/7$ and $\lambda_{\mathrm{CV}}=0.10$ multiplies the weight on active-dimension disagreement by about $2.8$. Two consequences follow: $\mathcal{L}_{\mathrm{CV}}$ is not redundant with flow matching even when both views are anchored to the same target, and its effect concentrates on the part of the velocity field that is sensitive to scene-camera pose. The identity holds for any two predictions and says nothing about the parameters that produce them; whether nominal-camera success is preserved once the weight rises is a property of the trained network, and Section~\ref{sec:exp_sim_main} measures it.

\textbf{$\mathcal{L}_{\mathrm{CV}}$ is a behavior-level consistency signal.}
Feature-level consistency does not guarantee action agreement (Appendix~\ref{app:feature_alternatives}), and final-action matching would require backpropagation through the numerical sampler, entangling local velocity disagreement with integration effects. We therefore regularize the velocity field itself, the quantity directly integrated to produce action chunks. Under the usual local Lipschitz assumption, velocity disagreement along the integration trajectory bounds the divergence of the integrated action chunks. Letting $a^*_0(s)$ and $a^*_p(s)$ denote the action chunks obtained by Euler-integrating $v_\theta$ from the nominal and perturbed views respectively, and $x_t^{(0)}$ a point on the nominal integration trajectory,
\begin{equation}
\mathbb{E}_{s,\varepsilon}\!
\left[\lVert a^*_0(s)-a^*_p(s)\rVert^2\right]
\leq
C\,\mathbb{E}_{s,t,\varepsilon}\!
\left[\lVert v_\theta(x_t^{(0)},t\mid o_0,l,q)-v_\theta(x_t^{(0)},t\mid o_p,l,q)\rVert^2\right],
\label{eq:action-divergence-bound}
\end{equation}
where $C$ depends on the sampler grid and the Lipschitz constant (Appendix~\ref{app:mechanism_proof}). The training-time loss $\mathcal{L}_{\mathrm{CV}}$ in~\eqref{eq:paired-cv-loss} evaluates this disagreement at the conditional flow-matching interpolation points $x_{t_k}=t_k\varepsilon_k+(1-t_k)a$, which coincide with the integration trajectory under the optimal flow-matching solution and otherwise serve as a tractable surrogate for the right-hand side. The loss therefore targets the velocity field whose integration yields rollout behavior, not an auxiliary representation.

\textbf{Shuffled pairs are the mechanism control.}
A key control keeps flow matching intact while breaking action equivalence only inside $\cvloss$. For a batch of $N$ pairs indexed by $i$, let $\pi$ be a derangement of $\{1,\ldots,N\}$. The shuffled cross-view term replaces the matched perturbed prediction with one from a different state:
\begin{equation}
  \mathcal{L}_{\mathrm{CV,shuf}}^{k}
  =
  \left\lVert
  v_{p,\pi(i)}^k[\,\cdot,:A_{\mathrm{act}}\,]
  -
  v_{0,i}^k[\,\cdot,:A_{\mathrm{act}}\,]
  \right\rVert^2 .
  \label{eq:shuffled-cv}
\end{equation}
Every view still receives its correct row-local flow-matching target. At the population level, suppressing shared arguments for readability,
\begin{equation}
\mathbb{E}_{s,s'}\lVert v_p(s')-v_0(s)\rVert^2
=
\mathbb{E}_{s}\lVert v_p(s)-v_0(s)\rVert^2
+2\,\mathrm{Tr}\big(\mathrm{Cov}_{s}(v_p(s),v_0(s))\big).
\label{eq:shuffled-decomposition}
\end{equation}
The second term is state-marginal covariance, not view disagreement for the same physical state (proof in Appendix~\ref{app:mechanism_proof}). When $v_p$ and $v_0$ retain meaningful state-dependent action information, this covariance is typically positive for meaningful state-dependent policies; minimizing the shuffled objective therefore actively opposes the supervised flow-matching signal by pulling per-state velocity predictions toward a common state-marginal mean. A shuffled-pair control thus tests whether the cross-view loss depends on true action-equivalent pairing or acts merely as a generic smoothing penalty.

\section{Experiments}
\label{sec:experiments}

We evaluate whether cross-view action-flow consistency improves scene-camera viewpoint robustness in simulation and on a real robot, using the LIBERO-Plus camera-perturbation benchmark and a held-out camera placement protocol respectively. The simulation experiments use same-state rerendered LIBERO pairs for training and the official LIBERO-Plus camera-perturbation benchmark for evaluation. The real-robot experiment uses synchronized multi-camera demonstrations for training and evaluates deployment from a held-out camera placement not used during data collection. In all experiments, the wrist-camera stream is masked during both training and evaluation: the wrist view is not subject to the scene-camera perturbation, so allowing the policy to use it would introduce an unperturbed visual shortcut that confounds the causal interpretation of scene-camera robustness. This isolates scene-camera attribution but makes our absolute success rates incomparable to standard LIBERO-Plus numbers that retain the wrist view; all comparisons reported here are made within this single-scene-RGB protocol against identically evaluated baselines.

\subsection{Experimental setup}
\label{sec:exp_setup}

\textbf{Simulation training data.}
We construct action-equivalent pairs from the original LIBERO demonstrations~\citep{liu2023libero} by resetting the simulator to stored MuJoCo states and rerendering the same physical state under a nominal scene camera and a perturbed scene camera (Figure~\ref{fig:LIBERO-rendered}). Each pair shares the task, language instruction, robot state, and demonstrated actions; only the scene-camera pose changes. Perturbed cameras are sampled from the same LIBERO-Plus C1/C2/C3 camera families~\citep{fei2025liberoplus} (distance and scale, spherical position, endpoint orientation) used by the evaluation track. Evaluation uses the separate LIBERO-Plus benchmark rollouts; evaluation task initial states and action labels are not used for training. We generate $338{,}575$ same-state pairs from $2{,}000$ episodes across $40$ tasks in \texttt{libero\_spatial}, \texttt{libero\_object}, \texttt{libero\_goal}, and \texttt{libero\_10}, and use the resulting pair data for the final simulation policies.

\textbf{Implementation details.}
All policies are fine-tuned from the $\pi_{0.5}$ checkpoint for $10{,}000$ steps with AdamW and a cosine learning-rate schedule. Paired runs use batches of $192$ synchronized pairs, $224\times224$ images, and action horizon $H=10$. We apply independent color jitter to the two views and no spatial augmentation. The proposed method uses $K=2$ flow samples per pair, $\lambda_{\mathrm{CV}}=0.10$, and $t_k\sim\mathrm{Beta}(2,3)$, with $\lambda_{\mathrm{CV}}$ ramped from zero over the first $500$ steps; the rationale for this flow-time concentration is given in Appendix~\ref{app:implementation}.

\textbf{Evaluation protocol.}
The LIBERO-Plus camera-perturbation benchmark~\citep{fei2025liberoplus} contains $1{,}599$ camera-perturbation task instances. It is divided into C1 camera distance and scale ($n=313$), C2 spherical camera position ($n=992$), and C3 endpoint camera orientation ($n=294$). We run $4{,}797$ trials, corresponding to $3$ full rollouts for each individual evaluation. We also report performance on the standard LIBERO benchmark (denoted as ID) with nominal cameras ($40$ tasks $\times$ $50$ trials $=2{,}000$ rollouts per policy); all policies use the single-scene-RGB, wrist-masked condition described above.

\textbf{Comparisons.}
The nominal-only baseline fine-tunes the open-sourced base $\pi_{0.5}$ policy on standard LIBERO nominal-camera data and measures fixed-camera brittleness. The naive mixed-camera SFT baseline fine-tunes on LIBERO-Plus public camera-perturbation data with ordinary flow matching. The primary same-data control is FM-only training on our same-state pair dataset, which uses the same rendered frames, architecture, optimizer, and $10{,}000$-step budget as our method but sets $\lambda_{\mathrm{CV}}=0$. Ablations then modify the cross-view term, augmentation geometry, gradient coupling, or pair identity.

\subsection{Main simulation results on LIBERO-Plus}
\label{sec:exp_sim_main}


\begin{table}[h]
  \centering
  \scriptsize
  \caption{Main simulation results on LIBERO-Plus. ID: standard LIBERO nominal-camera benchmark. Camera: LIBERO-Plus camera-perturbation aggregate. Both the FM-only same-state pair-data control and the proposed method are reported as mean $\pm$ std over training seeds 42--44. The FM-only pair-data row differs from the proposed method only in $\lambda_{\mathrm{CV}}=0$.}
  \label{tab:main-sim}
  \setlength{\tabcolsep}{3pt}
  \begin{tabular}{llccccc}
    \toprule
    Method & Training data / objective & ID & Camera & C1 & C2 & C3 \\
    \midrule
    Nominal-only baseline
      & nominal LIBERO, FM only
      & 91.1 & 16.8 & 1.1 & 13.2 & 45.7 \\
    Naive mixed-camera SFT
      & LIBERO-Plus camera data, FM only
      & 78.7 & 74.7 & 68.4 & 75.7 & 78.2 \\
    FM-only on same-state pairs, seeds 42--44
      & rerendered pairs, FM only
      & $93.6{\pm}0.4$
      & $79.8{\pm}0.8$
      & $72.7{\pm}0.2$
      & $79.9{\pm}1.4$
      & $87.2{\pm}0.7$ \\
    \midrule
    Proposed, seeds 42--44
      & rerendered pairs, bilateral $K=2$
      & $94.2{\pm}0.8$
      & $\mathbf{87.2{\pm}0.4}$
      & $\mathbf{81.9{\pm}1.0}$
      & $\mathbf{87.9{\pm}0.5}$
      & $\mathbf{90.6{\pm}0.9}$ \\
    \bottomrule
  \end{tabular}
\end{table}

Table~\ref{tab:main-sim} summarizes the main results; the full aggregate table and per-seed results are in Appendix~\ref{app:full_tables} and Appendix~\ref{app:seed_results}, respectively. The proposed method reaches $87.2{\pm}0.4\%$ on the camera track, with $+7.4$pp over the FM-only same-data control ($79.8{\pm}0.8\%$) and $+12.5$pp over naive mixed-camera SFT. Both pair-data variants are trained from identical data and differ only in $\lambda_{\mathrm{CV}}$; the gap is consistent across seeds (Appendix~\ref{app:seed_results}). Nominal-camera ID success does not fall: the two means differ by $0.6$pp, within the spread over training seeds. The camera-track gains are largest on C1 ($+9.2$pp) and C2 ($+8.0$pp), and remain positive on C3 ($+3.4$pp), where the proposed method reaches $90.6{\pm}0.9\%$ compared with $87.2{\pm}0.7\%$ for FM-only.

\subsection{Mechanism analysis and ablations}
\label{sec:exp_ablations}

\begin{table}[t]
  \centering
  \scriptsize
  \caption{Ablations and controls isolating the source of the camera-track improvement.}
  \label{tab:mechanism}
  \setlength{\tabcolsep}{3pt}
  \renewcommand{\arraystretch}{1.15}
  \begin{tabular}{>{\raggedright\arraybackslash}p{0.14\linewidth}>{\raggedright\arraybackslash}p{0.33\linewidth}cc>{\raggedright\arraybackslash}p{0.26\linewidth}}
    \toprule
    Question & Configuration & ID & Camera & Takeaway \\
    \midrule
    Data exposure only?
      & FM-only on same-state pairs (3 seeds)
      & $93.6{\pm}0.4$ & $79.8{\pm}0.8$
      & Camera-diverse pair data alone is insufficient. \\
    Wrong pairs?
      & Shuffled bilateral $K{=}2$ + $\mathrm{Beta}(2,3)$
      & 50.8 & 25.8
      & Stronger wrong cross-view targets severely damage behavior. \\
    Complete recipe
      & Proposed bilateral $K{=}2$ + $\mathrm{Beta}(2,3)$, seeds 42--44
      & $94.2{\pm}0.8$ & $87.2{\pm}0.4$
      & Best observed robustness/ID tradeoff. \\
    \bottomrule
  \end{tabular}
\end{table}

\begin{figure}[h]
  \centering
  \begin{subfigure}[t]{0.25\linewidth}
    \centering
    \includegraphics[width=\linewidth]{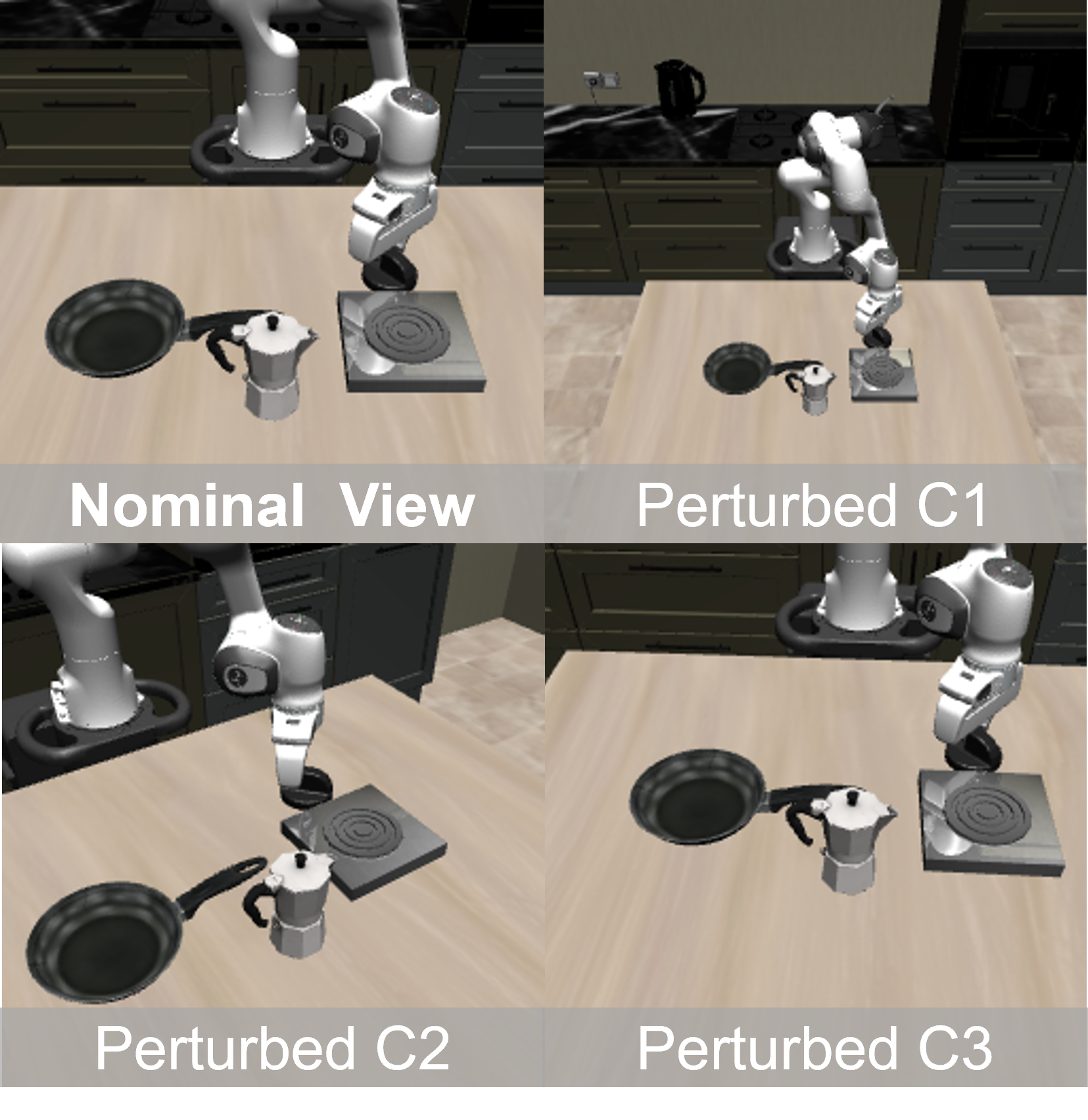}
    \caption{Perturbed Views.}
    \label{fig:LIBERO-rendered}
  \end{subfigure}
  \begin{subfigure}[t]{0.40\linewidth}
    \centering
    \includegraphics[width=\linewidth]{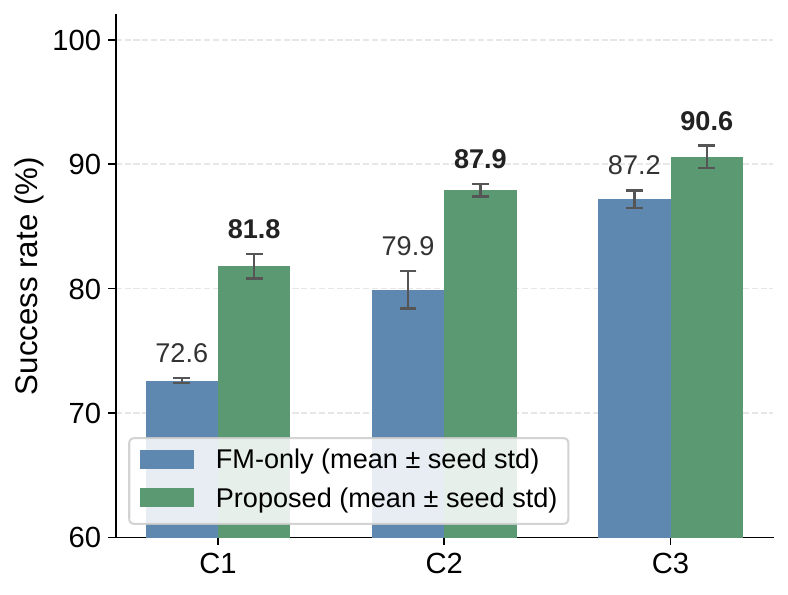}
    \caption{Per-category breakdown.}
    \label{fig:category-breakdown}
  \end{subfigure}
    \caption{LIBERO-Plus camera-track analysis. (a) Rendered nominal and perturbed scene-camera views. (b) Per-category success rates of FM-only vs.\ the proposed method, both shown as mean $\pm$ seed std over seeds 42--44.}
  \label{fig:category-and-shuffle}
\end{figure}

Table~\ref{tab:mechanism} isolates the source of the camera-track improvement, with supporting evidence in Figure~\ref{fig:category-and-shuffle}. Removing the cross-view term while preserving the same pair data drops camera-track performance from $87.2{\pm}0.4\%$ to $79.8{\pm}0.8\%$, a $7.4$pp gap attributable solely to $\lambda_{\mathrm{CV}}$ (per-seed results in Appendix~\ref{app:seed_results}). Because both variants share the same data, this gap reflects $\mathcal{L}_{\mathrm{CV}}$ alone, independent of any train/eval camera overlap.

A stronger test breaks action equivalence inside $\mathcal{L}_{\mathrm{CV}}$ by shuffling pairs across the batch, so that the loss compares velocity predictions from different physical states, each evaluated at the flow coordinate of its own state. Per-view flow-matching targets remain correct, but the cross-view signal becomes incoherent. Camera-track performance collapses from $87.2\%$ to $25.8\%$, and ID success falls with it. A generic smoothing of the velocity field would be indifferent to which state supplies the partner; the collapse is what Eq.~\eqref{eq:shuffled-decomposition} predicts when the predictions carry state information. The method requires same-state correspondence to be useful. Sensitivity to $\lambda_{\mathrm{CV}}$ is reported in Appendix~\ref{app:full_tables}. 

\subsection{Robustness beyond the paired support}
\label{sec:exp_support}

In Section~\ref{sec:exp_sim_main} the perturbed camera of every pair is drawn from the same distribution of poses that the evaluation track covers. Collecting demonstrations from that many viewpoints is impractical in deployment, so a deployed policy will meet camera poses that no pair covered. This section trains on pairs restricted to part of each perturbation axis and evaluates on the withheld part, and asks whether the cross-view term extrapolates beyond the pairs it saw.

Along each perturbation axis we hold out two bands of camera poses, one in the interior of the axis and one beyond its outer edge, and keep the rest. For elevation, the perturbed side of each training pair is fixed at $8^\circ$ relative to the nominal $0^\circ$ view, while the $15^\circ$ elevation is held out entirely. On the restricted pairs we retrain the FM-only control and the proposed objective with two training seeds (42 and 43), and we evaluate both on the full camera track. Each evaluation camera then belongs to one of four groups: inside the restricted support, in an interior held-out band (interpolation), beyond the outer edge (extrapolation), or at the held-out elevation (transfer). The spherical family is split into azimuth, elevation, and their compound. The extrapolation bands hold $300$, $297$, and $678$ evaluation rollouts per policy on azimuth, camera distance, and endpoint rotation, and the three-axis extrapolation composite pools all $1{,}275$.

The cross-view term improves the policy at viewpoints it never trained on. In both seeds it beats FM-only at extrapolation on azimuth, endpoint rotation, and camera distance, and the three-axis composite is positive in both. The margin is largest on azimuth. Inside the restricted support the two objectives are close, and on azimuth interpolation FM-only is slightly ahead. The term adds little where the training pairs already cover the viewpoint and adds most beyond them.

Elevation is the exception. Although the restricted training pairs contain a $0^\circ\!\leftrightarrow\!8^\circ$ elevation displacement, the consistency model performs worse than FM-only at the held-out $15^\circ$ elevation in both seeds. Thus, the positive extrapolation observed on azimuth, camera distance, and endpoint rotation does not extend uniformly to every camera factor. The same restricted-support setting also lowers nominal-camera success in both seeds, a cost absent under full-support training (Table~\ref{tab:main-sim}).

\subsection{Real-robot held-out camera evaluation}
\label{sec:exp_realrobot}

\begin{figure}[b]
  \centering
  \includegraphics[width=0.8\linewidth]{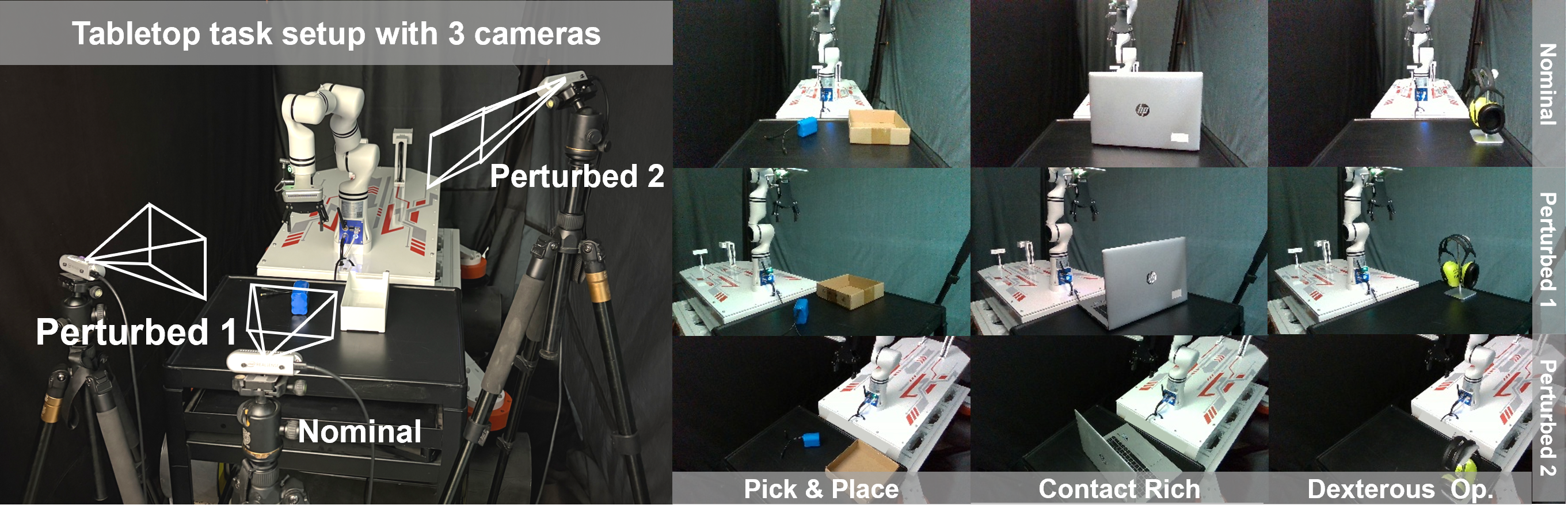}
  \caption{Real-robot synchronized multi-camera setup. Training demonstrations are collected with synchronized scene cameras $C_0,C_1,C_2$ to form action-equivalent pairs. At deployment, the policy receives only a single RGB scene image. The key evaluation condition places the deployment camera at a position never observed during training.}
  \label{fig:realrobot-setup}
\end{figure}

We collect synchronized multi-camera demonstrations for three tabletop tasks: Pick \& Place (blue battery into cardboard box), Contact Rich (closing a laptop lid), and Dexterous Operation (taking headphones off a stand). The three synchronized cameras at each timestep provide the real-hardware analog of the simulator-state-reset pairs in Section~\ref{sec:exp_sim_main}; independently collected trajectories from different placements are not used as pairs, since they are not action-equivalent.

\begin{figure}[h]
  \centering
  \includegraphics[width=0.7\linewidth]{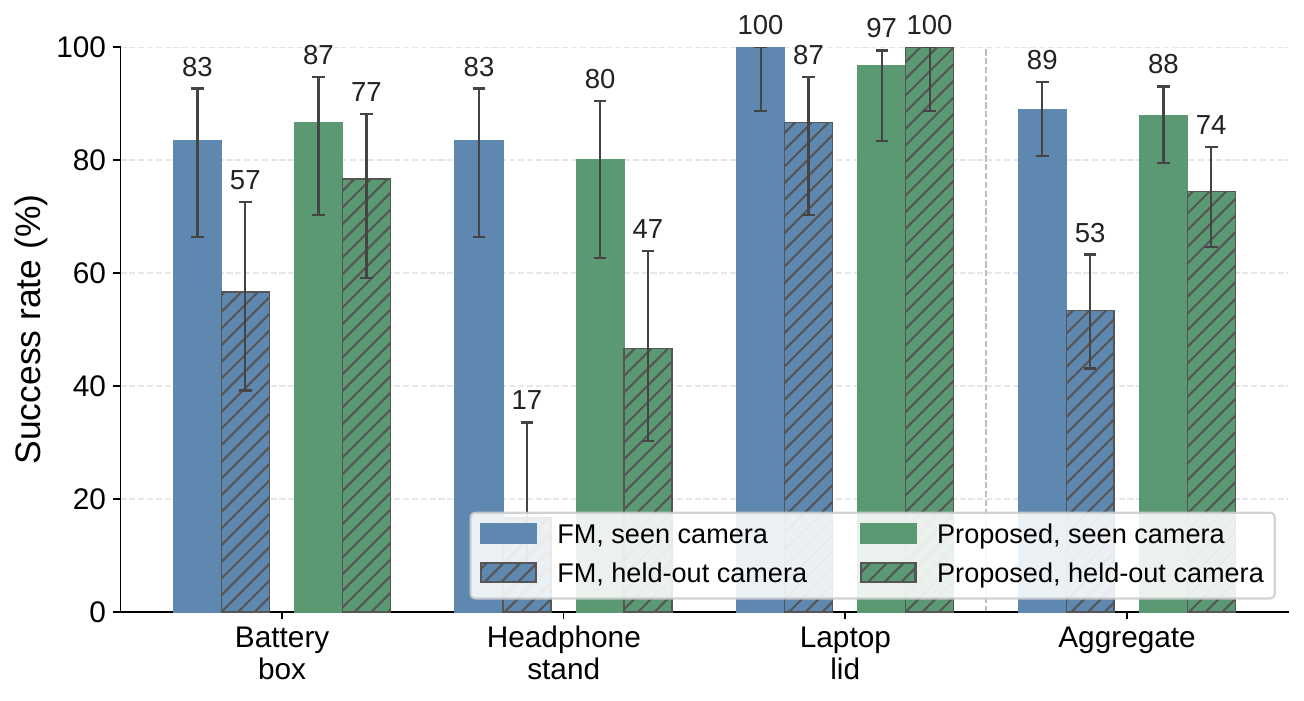}
  \caption{Task-level and aggregate success rates for seen and held-out camera placements.}
  \label{fig:realrobot-results}
\end{figure}

At deployment, each policy receives one RGB scene image, language, and proprioception, with the wrist stream masked. We evaluate from three seen training-camera placements ($C_0,C_1,C_2$ in Figure~\ref{fig:realrobot-setup}) and three held-out placements never used during data collection: $H_0$ (a farther distance extrapolation from $C_0$), $H_1$ (an azimuth rotation around the table center between $C_0$ and $C_1$), and $H_2$ (a higher and more top-down placement between $C_0$ and $C_2$); the six placements are visualized in Figure~\ref{fig:realrobot_views} (Appendix~\ref{app:realrobot_per_placement}). We run $10$ rollouts per task--placement cell ($90$ rollouts per method per condition).

Figure~\ref{fig:realrobot-results} reports aggregate and task-level results; full counts and Wilson intervals are in Appendix~\ref{app:realrobot}. On held-out cameras, FM-only multi-view training reaches $48/90$ aggregate successes, while the proposed objective reaches $67/90$ (descriptive two-proportion test, $p<0.005$; we treat per-task/per-placement counts as primary). The seen-camera condition shows that both methods can use views represented in the demonstration data ($80/90$ for FM-only, $79/90$ for the proposed method); the held-out-camera gap separates the two training signals. Gains appear on all three tasks but their magnitudes are uneven: the Laptop lid task transfers nearly losslessly under the proposed objective ($30/30$ held-out), the Battery box task improves from $17/30$ to $23/30$, and the Headphone stand task carries the largest absolute gap, $5/30$ to $14/30$. 

Two patterns emerge from the per-placement breakdown (Appendix~\ref{app:realrobot}). First, the held-out gain grows with the geometric aggressiveness of the placement: the largest task--placement margins occur on $H_2$, the highest and most top-down held-out view. Second, the Headphone stand task on $H_2$ is the single most informative case: FM-only succeeds in $0/10$ rollouts while the proposed objective succeeds in $4/10$, suggesting that the cross-view objective adds the most value where held-out viewpoint is most distant from training. Even with this gain, $16/30$ Headphone stand rollouts still fail: failures cluster around fine gripper--headphone alignment near contact, where held-out views occlude the thin band and small approach-angle errors cause mis-grasps (Appendix~\ref{app:realrobot}).

\section{Limitations}
\label{sec:discussion}

\textbf{Paired training data.} The method requires true action-equivalent pairs: in simulation by resetting MuJoCo states and rerendering, on hardware by synchronized cameras observing the same physical state at the same timestamp. This requirement excludes most existing single-camera robot datasets, in which different views typically come from different rollouts. Approximate pairs constructed from independently collected trajectories are not a safe shortcut: by the shuffled-pair decomposition in Section~\ref{sec:method}, enforcing consistency between non-equivalent states pushes per-state velocity predictions toward a state-marginal mean and actively opposes the supervised flow-matching signal, as confirmed empirically by the collapse to $25.8\%$ in Table~\ref{tab:mechanism}. Reducing this is our direction for future work (Appendix~\ref{app:future_work}).

\textbf{Observability under single RGB.} Cross-view consistency improves action agreement across views of the same state but cannot recover information not visually present in the input RGB; when held-out cameras occlude task-critical geometry or compress depth cues near contact, the objective has nothing to anchor against. This matches our inference-contract choice (Appendix~\ref{app:scope_comparison}), no depth, point clouds, RGB-D, or tactile sensing at deployment, and combining cross-view consistency with these is a direct extension.

\textbf{Scope of evaluation.} The hardware study is intentionally controlled: three tabletop tasks, static external scene cameras, three seen and three held-out placements, and $10$ rollouts per task---placement cell. It does not establish robustness beyond LIBERO-Plus C1/C2/C3 shifts, moving cameras, lighting changes, clutter, transparent/deformable objects, or mobile-base manipulation. Extending the evaluation protocol to these conditions, as well as transferring the action-flow consistency objective to discrete-token VLA backbones and other action-head families, is left to future work.

\section{Conclusion}
\label{sec:conclusion}

Camera motion can change a VLA's action predictions even when the physical task state is unchanged. For flow-based VLAs, we address this by regularizing the action-flow velocity field across action-equivalent views; on LIBERO-Plus and real-robot held-out-camera evaluations, this training-only objective improves over nominal-only, camera-diverse SFT, and same-data FM-only baselines while preserving single-scene-RGB inference, and the gain extends to camera poses that the training pairs never covered.

\acknowledgments{}

\bibliography{reference}

\appendix
\clearpage

\section{Dataset and Evaluation Details}
\label{app:data_eval}

\subsection{Same-state simulation pair construction}
We use original LIBERO demonstrations as the source for same-state pair construction. For a selected demonstration timestep $t$, we load the corresponding environment, reset the simulator to the stored MuJoCo state, call \texttt{sim.forward()}, and render two scene-camera observations: a nominal view and a perturbed view. The row metadata records suite, task, episode, timestep, state hash, action-chunk hash, camera category, and camera parameters. A matched pair is valid only if the two slots share the same simulator state, action chunk, language instruction, and proprioceptive state.

The generated dataset contains $338{,}575$ same-state pairs from $2{,}000$ episodes across $40$ tasks in \texttt{libero\_spatial}, \texttt{libero\_object}, \texttt{libero\_goal}, and \texttt{libero\_10}. Camera categories follow LIBERO-Plus: C1 distance/scale, C2 spherical camera position, and C3 endpoint orientation. The real-robot dataset follows the same logical structure, but exact same-state observations are obtained from synchronized cameras rather than simulator reset.

\subsection{Train/evaluation separation}
The simulation training pairs are generated from original LIBERO demonstration states, with perturbed scene cameras drawn from the same C1/C2/C3 camera families and pose distribution used by the LIBERO-Plus camera track. Evaluation is performed on the separate official LIBERO-Plus benchmark rollouts; evaluation task initial states and action labels are not used for training, and we do not evaluate on a held-out split of the rendered training-pair dataset itself. The full-support experiment therefore tests generalization to separate benchmark rollouts under the same camera-perturbation family, while Section~\ref{sec:exp_support} explicitly tests camera poses outside the restricted training support. No wrist image is used in any reported simulation policy, because LIBERO-Plus perturbs the scene camera but leaves wrist observations stable; retaining wrist images would no longer isolate scene-camera robustness. All training and evaluation use a 6$\times$ NVIDIA RTX PRO 6000 Server Edition GPU setup; detailed hyperparameters are documented in the open-sourced code repository linked from the project page.

\subsection{Rollout metrics}
All rollout success rates are binomial proportions. Single-seed baseline and ablation rows are reported with rollout-level Wilson 95\% confidence intervals. The primary FM-only same-state-pair control and the proposed method are reported as mean $\pm$ standard deviation over training seeds 42--44; their seed-level results are given in Appendix~\ref{app:seed_results}.

\section{Full Simulation Tables}
\label{app:full_tables}

\begin{table}[h]
  \centering
  \scriptsize
  \caption{Full simulation table. Single-seed baseline rows are reported with rollout-level Wilson 95\% confidence intervals on the camera track; ID is over $2{,}000$ rollouts per policy. The FM-only same-state pair control and the proposed method are both reported as mean $\pm$ standard deviation across training seeds 42--44; per-seed values for both are in Table~\ref{tab:app_per_seed_fm_proposed}.}
  \label{tab:app_full_sim}
  \setlength{\tabcolsep}{8pt}
  \begin{tabular}{lccccc}
    \toprule
    Method & ID & Camera & C1 & C2 & C3 \\
    \midrule
    Nominal-only baseline & $91.1$ & $16.8$ [15.8, 17.9] & $1.1$ [0.6, 1.9] & $13.2$ [12.0, 14.5] & $45.7$ [42.4, 49.0] \\
    Naive mixed-camera SFT & $78.7$ & $74.7$ [73.4, 75.9] & $68.4$ [65.3, 71.3] & $75.7$ [74.1, 77.2] & $78.2$ [75.4, 80.8] \\
    FM-only on same-state pairs, seeds 42--44 & $93.6{\pm}0.4$ & $79.8\pm0.8$ & $72.7\pm0.2$ & $79.9\pm1.4$ & $87.2\pm0.7$ \\
    Proposed, seeds 42--44 & $94.2\pm0.8$ & $87.2\pm0.4$ & $81.9\pm1.0$ & $87.9\pm0.5$ & $90.6\pm0.9$ \\
    \bottomrule
  \end{tabular}
\end{table}

\begin{table}[h]
  \centering
  \normalsize 
  \caption{Extended ablations on the camera track. The $\lambda_{\mathrm{CV}}$ sweep is run in an exploratory single-sample, independent-augmentation regime, distinct from the final $K{=}2$ configuration with no spatial augmentation in which $\lambda_{\mathrm{CV}}{=}0.10$ was selected; its absolute numbers are therefore not directly comparable to Table~\ref{tab:main-sim}. ID is omitted for these exploratory runs.}
  \label{tab:app_extended_ablations}
  \setlength{\tabcolsep}{10pt}
  \begin{tabular}{lc}
    \toprule
    Configuration & Camera \\
    \midrule
    \multicolumn{2}{l}{\emph{$\lambda_{\mathrm{CV}}$ sensitivity, single-sample symmetric, independent spatial augmentation}} \\
    \quad $\lambda=0.05$ & $81.1$ [80.0, 82.2] \\
    \quad $\lambda=0.10$ & $80.9$ [79.8, 82.0] \\
    \quad $\lambda=0.20$ & $73.6$ [72.5, 74.7] \\
    \quad $\lambda=0.50$ & $68.1$ [67.0, 69.2] \\
    \midrule
    \multicolumn{2}{l}{\emph{Pair augmentation regime, single-sample symmetric, $\lambda=0.10$}} \\
    \quad Independent spatial augmentation & $80.9$ [79.8, 82.0] \\
    \quad Pair-consistent spatial handling & $84.9$ [83.9, 85.9] \\
    \midrule
    \multicolumn{2}{l}{\emph{Matched vs. shuffled, final $K=2$ recipe}} \\
    \quad Shuffled bilateral $K=2$ + $\mathrm{Beta}(2,3)$ & $25.8$ [24.4, 27.2] \\
    \bottomrule
  \end{tabular}
\end{table}
On the camera track, $\lambda_{\mathrm{CV}}{=}0.05$ and $0.10$ lie within overlapping rollout-level Wilson intervals ($81.1$ [$80.0,82.2$] vs.\ $80.9$ [$79.8,82.0$]); camera-track performance is thus insensitive to $\lambda_{\mathrm{CV}}$ across $[0.05,0.10]$ in this regime, and we carry $\lambda_{\mathrm{CV}}{=}0.10$ to the final configuration.

\section{Seed-Level Results for FM-only and Proposed}
\label{app:seed_results}

We train both the FM-only same-state pair control ($\lambda_{\mathrm{CV}}=0$) and the proposed bilateral $K=2$ + $\mathrm{Beta}(2,3)$ configuration with three independent training seeds (42, 43, 44) and evaluate each on the full LIBERO-Plus camera-perturbation track. The two methods share identical training data, architecture, optimizer, and step budget; they differ only in $\lambda_{\mathrm{CV}}$.

\begin{table}[!h]
  \centering
  \scriptsize
  \caption{Per-seed camera-track results for the FM-only same-state pair control and the proposed bilateral $K=2$ + $\mathrm{Beta}(2,3)$ configuration. Camera denotes the LIBERO-Plus camera-perturbation aggregate; ID means over seeds are in Table~\ref{tab:app_full_sim}. Means and standard deviations are computed from exact rollout counts, so they may differ in the last digit from the average of the rounded per-seed values.}
  \label{tab:app_per_seed_fm_proposed}
  \setlength{\tabcolsep}{6pt}
  \begin{tabular}{llcccc}
    \toprule
    Method & Seed & Camera & C1 & C2 & C3 \\
    \midrule
    \multirow{4}{*}{FM-only on same-state pairs}
      & 42 & 79.5 & 72.8 & 79.4 & 86.6 \\
      & 43 & 80.7 & 72.5 & 81.5 & 87.0 \\
      & 44 & 79.2 & 72.6 & 78.7 & 88.0 \\
      & Mean $\pm$ std.\ & $79.8\pm0.8$ & $72.7\pm0.2$ & $79.9\pm1.4$ & $87.2\pm0.7$ \\
    \midrule
    \multirow{4}{*}{Proposed, bilateral $K=2$ + $\mathrm{Beta}(2,3)$}
      & 42 & 86.9 & 80.7 & 88.1 & 89.7 \\
      & 43 & 87.6 & 82.2 & 88.3 & 90.7 \\
      & 44 & 87.2 & 82.6 & 87.3 & 91.5 \\
      & Mean $\pm$ std.\ & $87.2\pm0.4$ & $81.9\pm1.0$ & $87.9\pm0.5$ & $90.6\pm0.9$ \\
    \bottomrule
  \end{tabular}
\end{table}

Across seeds, the proposed method's per-seed Camera score (lowest $86.9\%$, highest $87.6\%$) is strictly above the FM-only same-state pair control's per-seed Camera score (lowest $79.2\%$, highest $80.7\%$). The matched-seed gap ranges from $+6.9$pp (seed 43) to $+8.0$pp (seed 44); moreover, the lowest proposed seed still exceeds the highest FM-only seed by $+6.2$pp. Category-level gains are positive on all three camera perturbation types: $+9.2$pp on C1, $+8.0$pp on C2, and $+3.4$pp on C3. On ID the two means differ by $0.6$pp, within the spread over training seeds, so the cross-view term does not cost nominal-camera success. A paired cluster bootstrap over the $1{,}599$ camera-task instances ($100{,}000$ resamples) puts the mean camera-track gain between $6.5$pp and $8.3$pp.

\section{Real-Robot Experiments and Full Tables}
\label{app:realrobot}

\subsection{Hardware setup and synchronization}
Real-robot demonstrations are collected with three synchronized external scene cameras $C_0,C_1,C_2$. For a demonstration timestep $t$, the tuple $(C_0[t],C_1[t],C_2[t],l,q[t],a[t])$ is treated as action-equivalent supervision because all cameras observe the same physical state at the same timestamp. Trajectories collected independently from different camera placements are not used for the cross-view term, because they do not guarantee same-state action equivalence. Table~\ref{tab:realrobot-details} summarizes the real-robot data collection, training, and checkpoint-selection protocol. Held-out camera placements are reserved strictly for evaluation and are not used for training or checkpoint selection. 

\begin{table}[h]
  \centering
  \small
  \caption{Real-robot data collection and training details. Held-out camera placements are used only for evaluation and never for training or checkpoint selection.}
  \label{tab:realrobot-details}
  \setlength{\tabcolsep}{4pt}
  \renewcommand{\arraystretch}{1.12}
  \begin{tabular}{p{0.34\linewidth}p{0.58\linewidth}}
    \toprule
    Item & Setting \\
    \midrule
    Robot and action space
      & RealMan RM-75 7-DoF manipulator; delta end-effector pose actions, controlled at 15 Hz. \\

    Tasks
      & Battery box; Headphone stand; Laptop lid. \\

    Training cameras
      & Three synchronized external scene cameras, denoted $C_0$, $C_1$, and $C_2$. \\

    Held-out cameras
      & $H_0$, $H_1$, and $H_2$; used only for evaluation. \\

    Demonstrations
      & 70 demonstrations per task, collected with domain randomization. \\

    Training samples
      & Each demonstration timestep is recorded by three synchronized cameras. Action-equivalent training pairs are constructed from these synchronized views. \\

    Image preprocessing
      & Intel RealSense D435i cameras; RGB images are captured at $640 \times 480$ resolution and 15 fps. \\

    Initialization
      & $\pi_{0.5}$ checkpoint. \\

    FM-only training
      & Same real-robot multi-view data and training budget as the proposed method, with $\lambda_{\mathrm{CV}}=0$. Training steps, batch size, optimizer, and learning-rate schedule follow the simulation setting. \\

    Proposed training
      & Cross-view action-flow consistency on synchronized action-equivalent pairs. Hyperparameters follow the simulation setting. \\

    Checkpoint selection
      & Final checkpoint after training; no held-out-camera rollouts are used for checkpoint selection. \\

    Evaluation budget
      & 10 rollouts per task--placement cell; 90 seen-camera and 90 held-out-camera rollouts per method. \\
    \bottomrule
  \end{tabular}
\end{table}

\subsection{Tasks and rollout budget}
The three tasks are: blue battery into cardboard box, take the headphones off the stand, and close laptop lid. For each method and deployment camera condition, we run $30$ rollouts per task ($3$ camera placements $\times$ $10$ rollouts each). With two core methods and two deployment camera conditions, the total core evaluation is $2\times3\times2\times30=360$ rollouts.

\begin{table}[h]
  \centering
  \scriptsize
  \caption{Real-robot task-level breakdown. Each cell reports success rate with Wilson 95\% confidence interval and success count.}
  \label{tab:app_real_task}
  \setlength{\tabcolsep}{3pt}
  \begin{tabular}{llcccc}
    \toprule
    Method & Camera & Battery box & Headphone stand & Laptop lid & Aggregate \\
    \midrule
    FM, multi-view mixed & Seen
      & 83.3 [66.4, 92.7] (25/30)
      & 83.3 [66.4, 92.7] (25/30)
      & 100 [88.7, 100] (30/30)
      & 88.9 [80.7, 93.8] (80/90) \\
    FM, multi-view mixed & Held-out
      & 56.7 [38.9, 72.9] (17/30)
      & 16.7 [7.3, 33.6] (5/30)
      & 86.7 [70.3, 94.7] (26/30)
      & 53.3 [43.1, 63.3] (48/90) \\
    Proposed & Seen
      & 86.7 [70.3, 94.7] (26/30)
      & 80.0 [62.7, 90.5] (24/30)
      & 96.7 [83.3, 99.4] (29/30)
      & 87.8 [79.4, 93.0] (79/90) \\
    Proposed & Held-out
      & 76.7 [59.1, 88.2] (23/30)
      & 46.7 [30.2, 63.9] (14/30)
      & 100 [88.7, 100] (30/30)
      & 74.4 [64.6, 82.4] (67/90) \\
    \bottomrule
  \end{tabular}
\end{table}

\subsection{Per-placement task-level results}
\label{app:realrobot_per_placement}

Table~\ref{tab:app_real_per_placement} expands the aggregate results in Table~\ref{tab:app_real_task} to per-placement task-level success counts. $H_0$ extrapolates camera distance away from $C_0$; $H_1$ rotates azimuthally around the table center between $C_0$ and $C_1$; $H_2$ raises and tilts the view between $C_0$ and $C_2$.

\begin{table}[h]
  \centering
  \scriptsize
  \caption{Per-placement task-level success counts (out of $10$ rollouts per cell). Seen placements are $C_0,C_1,C_2$; held-out placements are $H_0,H_1,H_2$.}
  \label{tab:app_real_per_placement}
  \setlength{\tabcolsep}{4pt}
  \begin{tabular}{llcccccc}
    \toprule
    & & \multicolumn{3}{c}{Seen} & \multicolumn{3}{c}{Held-out} \\
    \cmidrule(lr){3-5} \cmidrule(lr){6-8}
    Task & Method & $C_0$ & $C_1$ & $C_2$ & $H_0$ & $H_1$ & $H_2$ \\
    \midrule
    \multirow{2}{*}{Battery box}
      & FM       & 9 & 8 & 8 & 8 & 6 & 3 \\
      & Proposed & 10 & 9 & 7 & 9 & 8 & 6 \\
    \midrule
    \multirow{2}{*}{Laptop lid}
      & FM       & 10 & 10 & 10 & 9 & 9 & 8 \\
      & Proposed & 10 & 10 & 9 & 10 & 10 & 10 \\
    \midrule
    \multirow{2}{*}{Headphone stand}
      & FM       & 8 & 9 & 8 & 3 & 2 & 0 \\
      & Proposed & 7 & 9 & 8 & 5 & 5 & 4 \\
    \bottomrule
  \end{tabular}
\end{table}

\begin{figure}[t]
  \centering
  \begin{subfigure}[t]{0.32\linewidth}
    \centering
    \includegraphics[width=\linewidth]{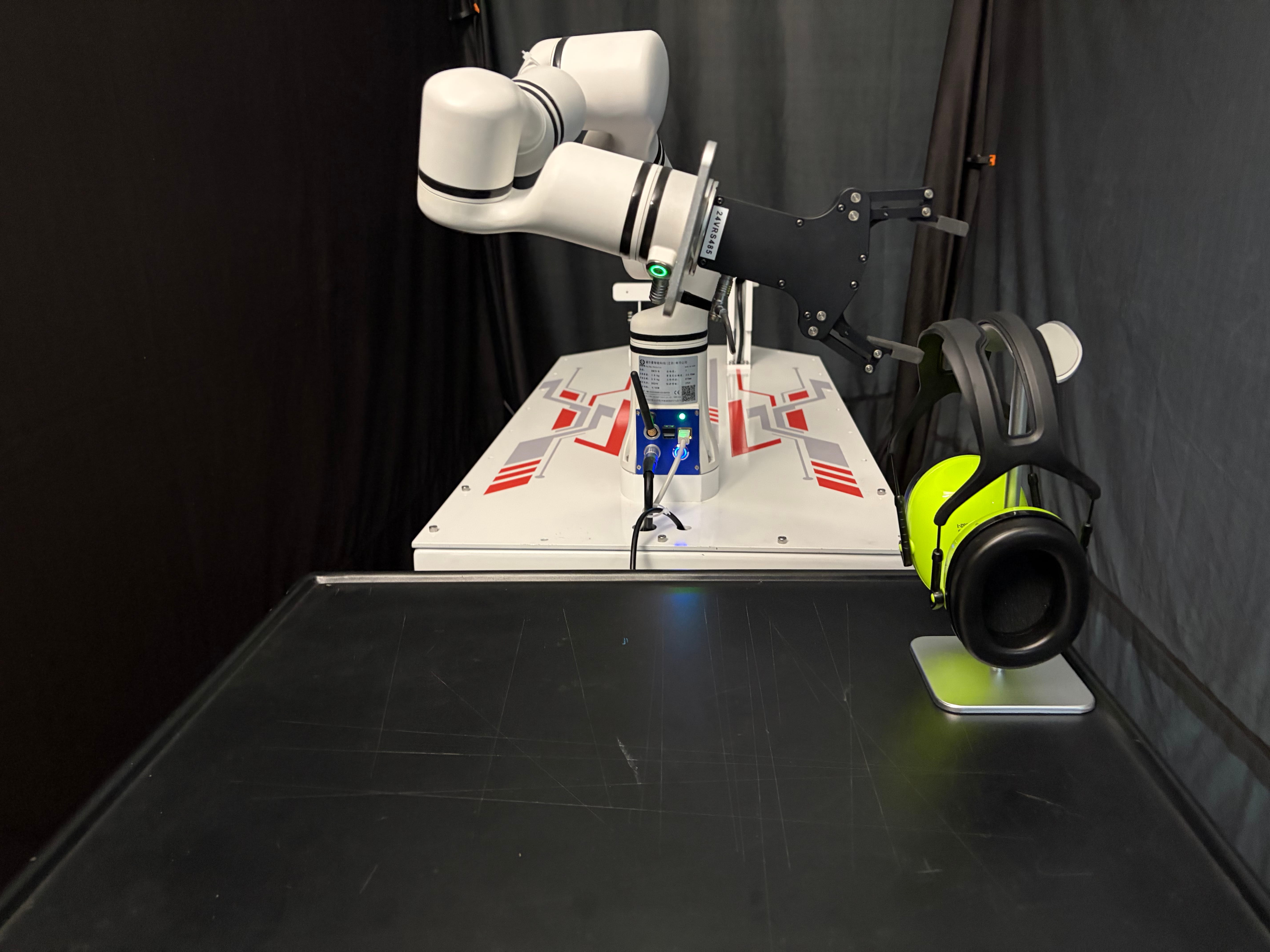}
    \caption{$C_0$ (seen)}
    \label{fig:realrobot_views_C0}
  \end{subfigure}
  \begin{subfigure}[t]{0.32\linewidth}
    \centering
    \includegraphics[width=\linewidth]{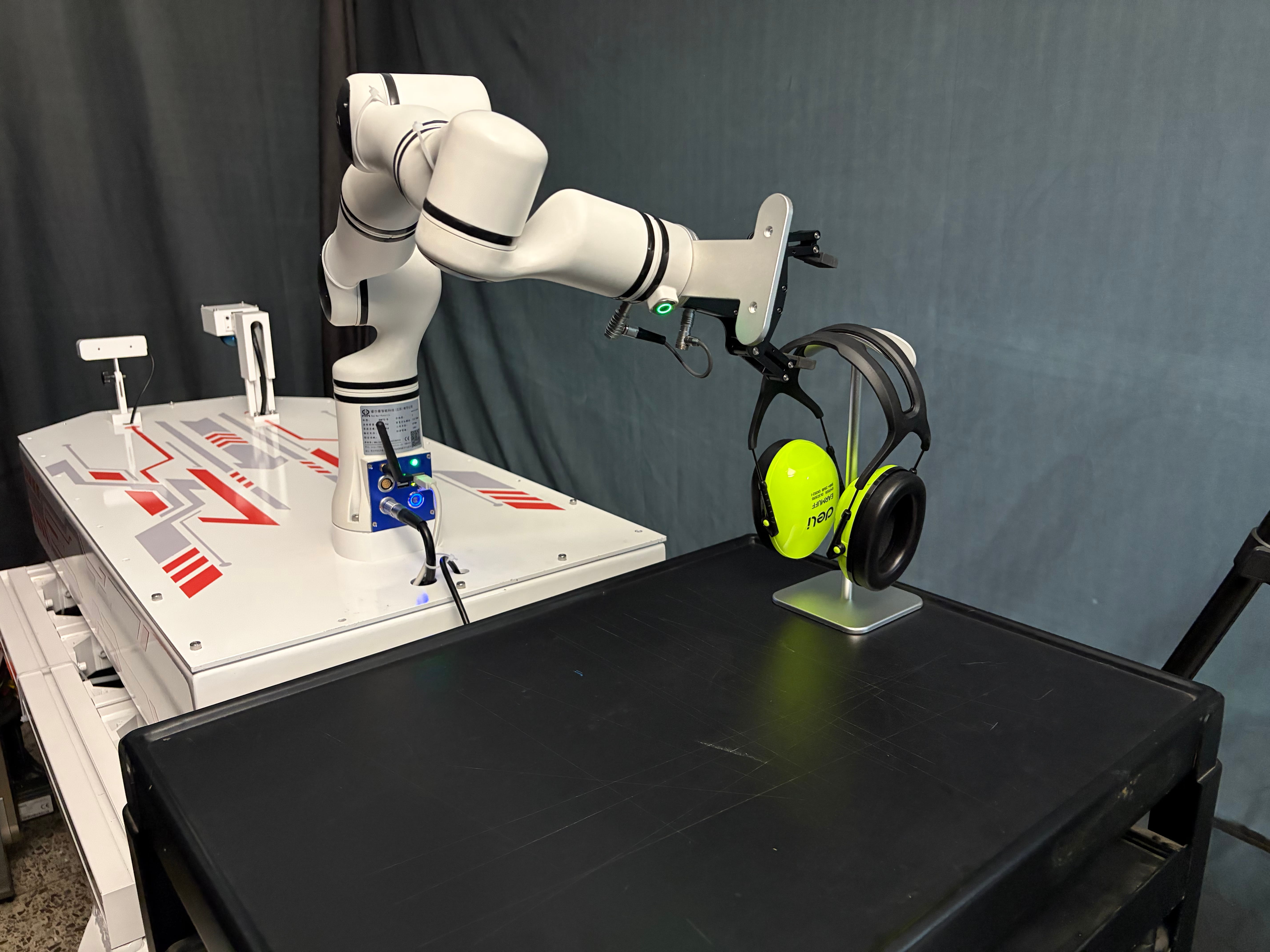}
    \caption{$C_1$ (seen)}
    \label{fig:realrobot_views_C1}
  \end{subfigure}
  \begin{subfigure}[t]{0.32\linewidth}
    \centering
    \includegraphics[width=\linewidth]{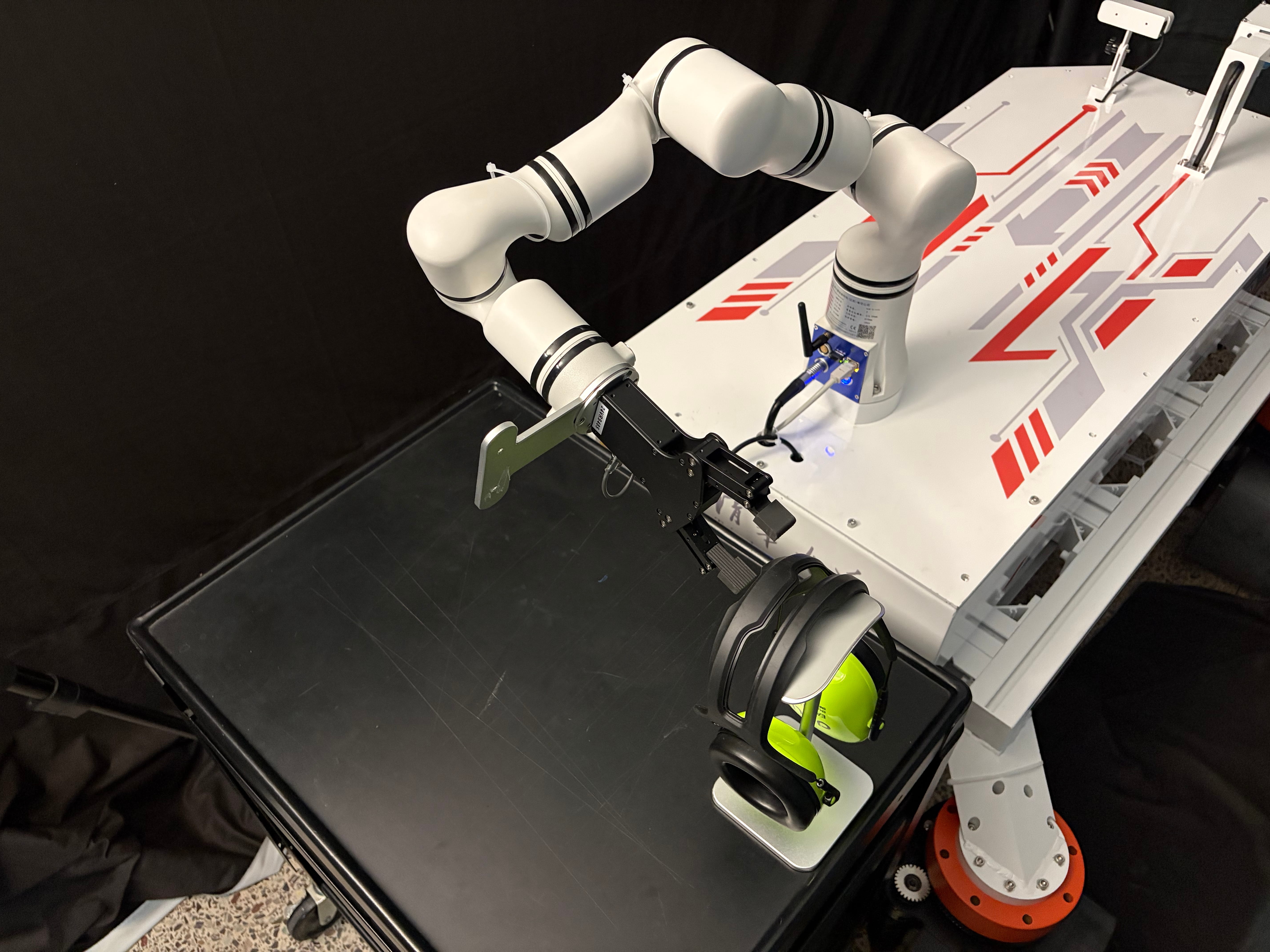}
    \caption{$C_2$ (seen)}
    \label{fig:realrobot_views_C2}
  \end{subfigure}
  \\[2pt]
  \begin{subfigure}[t]{0.32\linewidth}
    \centering
    \includegraphics[width=\linewidth]{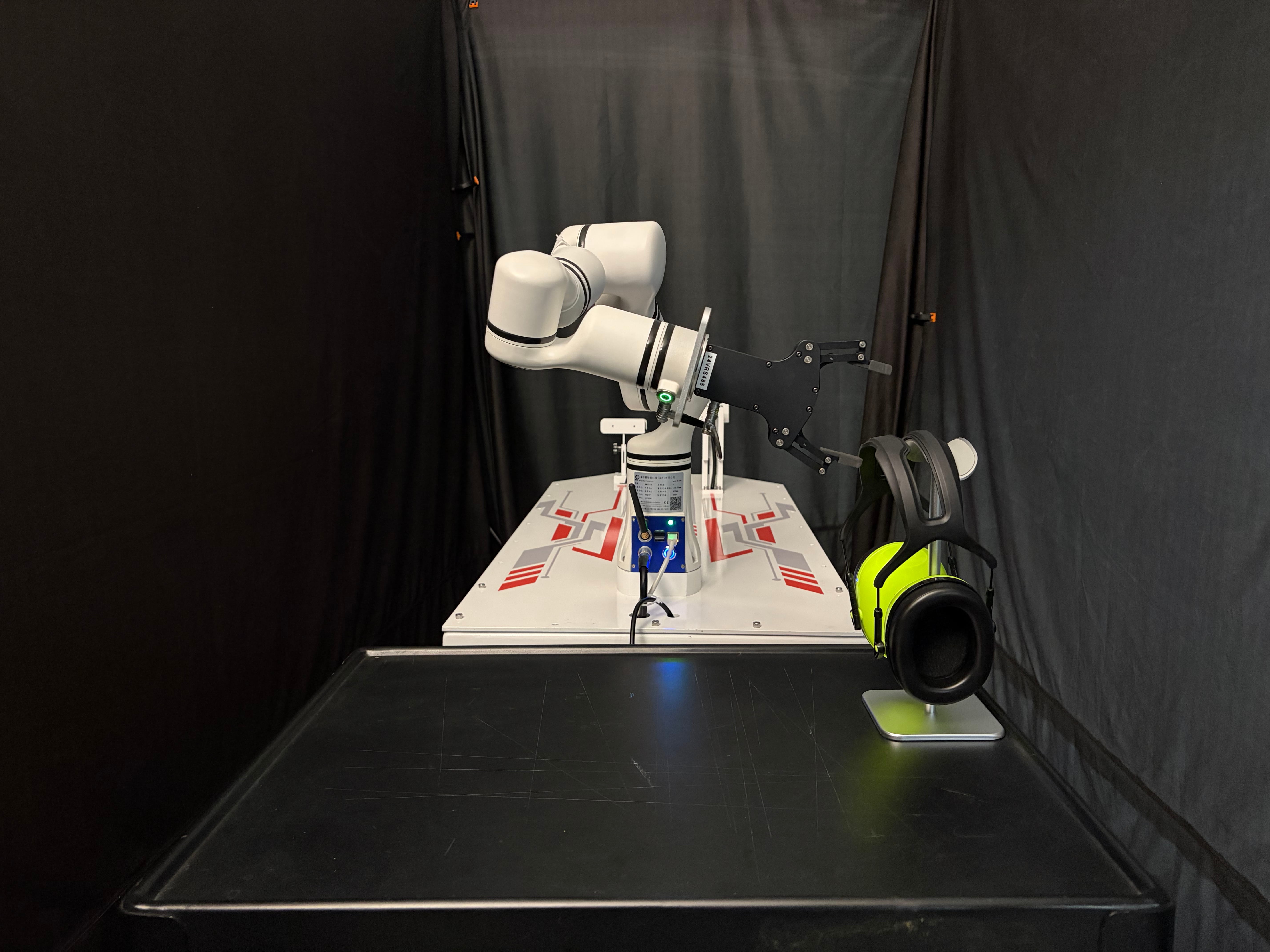}
    \caption{$H_0$ (held-out, farther)}
    \label{fig:realrobot_views_H0}
  \end{subfigure}
  \begin{subfigure}[t]{0.32\linewidth}
    \centering
    \includegraphics[width=\linewidth]{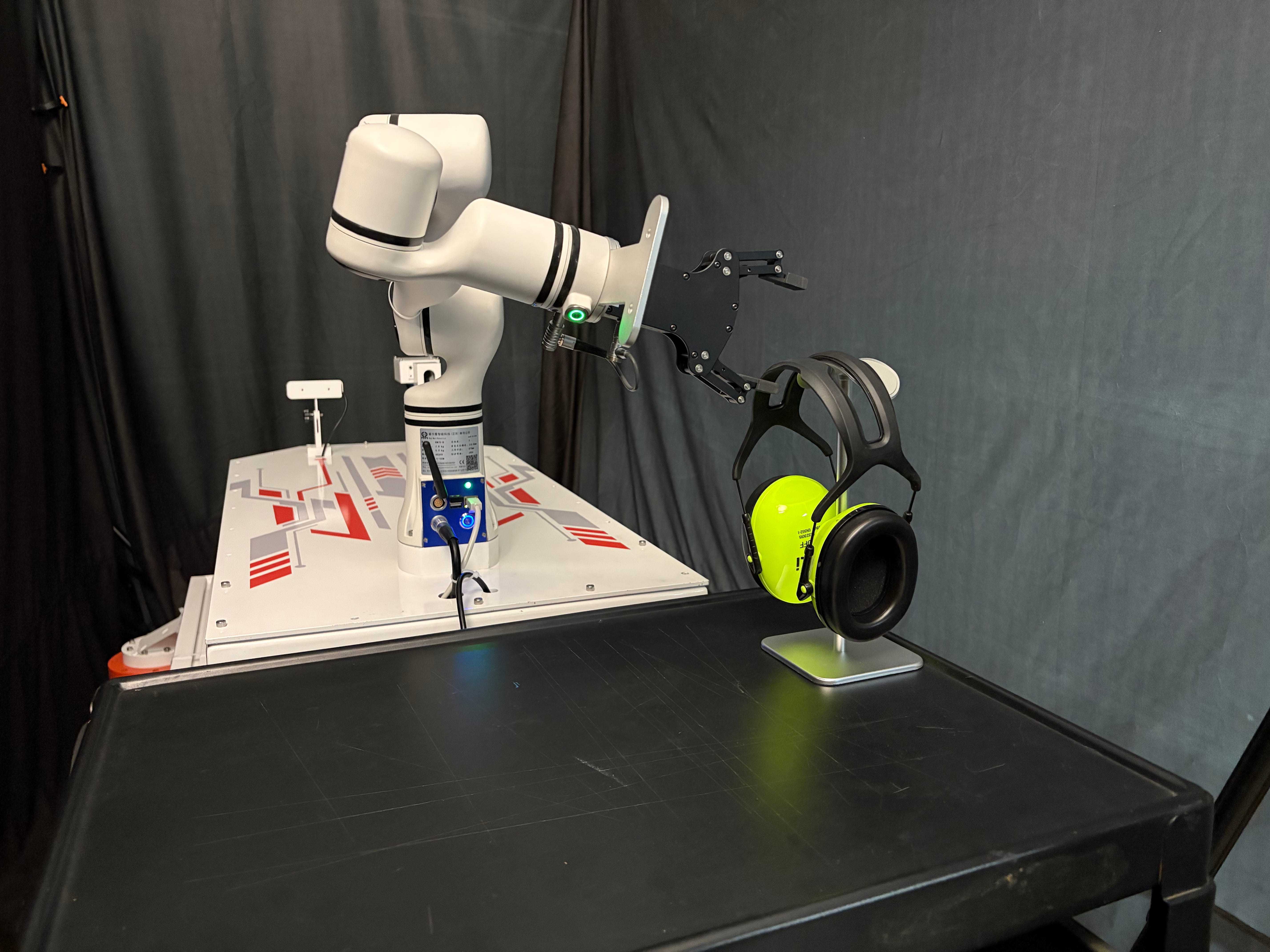}
    \caption{$H_1$ (held-out, azimuth)}
    \label{fig:realrobot_views_H1}
  \end{subfigure}
  \begin{subfigure}[t]{0.32\linewidth}
    \centering
    \includegraphics[width=\linewidth]{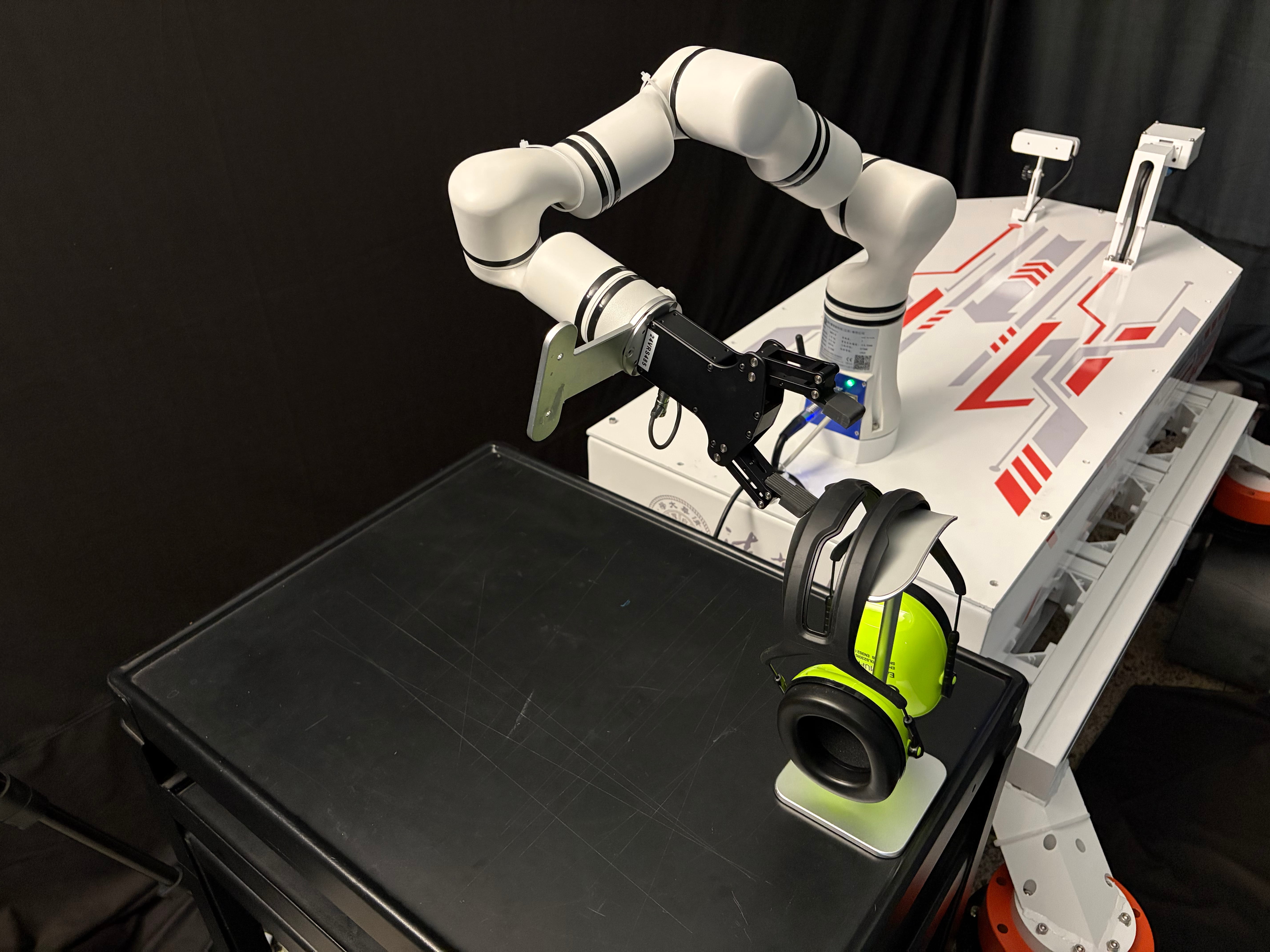}
    \caption{$H_2$ (held-out, higher/top-down)}
    \label{fig:realrobot_views_H2}
  \end{subfigure}
  \caption{Scene-camera views used for the real-robot held-out evaluation. All six images show the same physical state of the Headphone stand task moments before grasp contact, with the gripper positioned a few centimeters above the headphone band. Top row: training-camera placements $C_0,C_1,C_2$ used for both demonstration collection and cross-view supervision. Bottom row: held-out placements $H_0,H_1,H_2$ used only at evaluation. The held-out views progressively depart from the training cameras: $H_0$ extrapolates camera distance, $H_1$ rotates azimuthally, and $H_2$ combines a higher position with a more top-down angle. }
  \label{fig:realrobot_views}
\end{figure}

Two patterns emerge, both visually consistent with Figure~\ref{fig:realrobot_views}. First, the proposed method's advantage scales with the geometric aggressiveness of the held-out placement: aggregating across tasks, FM-only and the proposed method differ by $4$ rollouts on $H_0$ ($20/30$ vs $24/30$), $6$ on $H_1$ ($17/30$ vs $23/30$), and $9$ on $H_2$ ($11/30$ vs $20/30$). Second, the Headphone stand task on $H_2$ is the most extreme case: FM-only never succeeds ($0/10$), while the proposed method succeeds in $4/10$ rollouts---consistent with the structural-observability limitation discussed in Section~\ref{sec:discussion}, since $H_2$ combines the most aggressive viewpoint change with the task that has the smallest visually-recoverable contact geometry.

\subsection{Real-robot reporting rule}
The main real-robot claim should be based on held-out-camera success, the seen-to-held-out drop, and the number of tasks that improve. We report real-robot results as success counts and Wilson intervals. Qualitative rollout frames can be shown in the real-robot figure, videos are supplementary evidence and visible at project page.

\subsection{Qualitative failure analysis}
The held-out-camera degradation is task-dependent. The headphone-stand task shows the largest drop for both methods: FM-only multi-view mixed training reaches only $5/30$ held-out successes, while the proposed method improves to $14/30$ but remains far from saturated. Qualitatively, this task requires fine gripper--headphone alignment and is sensitive to depth ambiguity under novel viewpoints. The laptop-lid task is visually larger and more constrained by contact geometry, and both methods transfer more reliably, with the proposed method reaching $30/30$ held-out successes. These patterns suggest that cross-view action-flow consistency helps substantially under camera shift, but does not remove failures caused by poor observability, thin objects, or viewpoint-induced depth ambiguity.

\section{Mechanism Proof}
\label{app:mechanism_proof}
This section expands the mechanism statements used in the main paper.
It is not required to understand the main experimental claim; the main
text contains the key equations and empirical controls.

\begin{proposition}[Action-chunk divergence bound]
\label{prop:divergence-bound}
Let $0=t_0<t_1<\cdots<t_N=1$ be a uniform Euler grid with step size
$\Delta t=1/N$. Let $x^{(0)}_{t_n}$ and $x^{(p)}_{t_n}$ denote the
Euler-integration trajectories for the nominal and perturbed views
respectively, both initialized at the same noise
$x^{(0)}_{t_N}=x^{(p)}_{t_N}=\varepsilon$.
If $v_\theta(\cdot,t\mid o,l,q)$ is $L$-Lipschitz in its first argument
uniformly in $t$, then
\begin{equation}
  \bigl\lVert x^{(0)}_{t_0}-x^{(p)}_{t_0}\bigr\rVert^2
  \;\leq\;
  e^{2L}\sum_{n=1}^{N}\Delta t\,
  \bigl\lVert
    v_\theta(x^{(0)}_{t_n},t_n\mid o_0,l,q)
    -v_\theta(x^{(0)}_{t_n},t_n\mid o_p,l,q)
  \bigr\rVert^2.
  \label{eq:divergence-bound-full}
\end{equation}
\end{proposition}

\begin{proof}
Define $\delta_n=\lVert x^{(0)}_{t_n}-x^{(p)}_{t_n}\rVert$ and
\[
  e_n=\bigl\lVert
    v_\theta(x^{(0)}_{t_n},t_n\mid o_0,l,q)
    -v_\theta(x^{(0)}_{t_n},t_n\mid o_p,l,q)
  \bigr\rVert.
\]
The Euler update from $t_n$ to $t_{n-1}$ is
$x^{(\cdot)}_{t_{n-1}}=x^{(\cdot)}_{t_n}
-\Delta t\,v_\theta(x^{(\cdot)}_{t_n},t_n\mid o_\cdot,l,q)$.
Subtracting the two trajectories and applying the triangle inequality:
\begin{equation}
  \delta_{n-1}
  \;\leq\;
  \delta_n
  +\Delta t\,\bigl\lVert
    v_\theta(x^{(0)}_{t_n},t_n\mid o_0)
    -v_\theta(x^{(p)}_{t_n},t_n\mid o_p)
  \bigr\rVert.
  \label{eq:triangle-step}
\end{equation}
Inserting $\pm\,v_\theta(x^{(0)}_{t_n},t_n\mid o_p)$ into the norm
and applying the $L$-Lipschitz condition to the second difference:
\begin{equation}
  \delta_{n-1}\;\leq\;(1+L\Delta t)\,\delta_n+\Delta t\,e_n.
  \label{eq:gronwall-step}
\end{equation}
Unrolling~\eqref{eq:gronwall-step} from $n=N$ down to $n=1$
with $\delta_N=0$ gives
\begin{equation}
  \delta_0
  \;\leq\;
  \sum_{n=1}^{N}(1+L\Delta t)^{n-1}\,\Delta t\,e_n
  \;\leq\;
  e^{L}\sum_{n=1}^{N}\Delta t\,e_n,
  \label{eq:unrolled}
\end{equation}
where $(1+L\Delta t)^{n-1}\leq(1+L\Delta t)^{N}\leq e^{LN\Delta t}=e^{L}$.
Squaring~\eqref{eq:unrolled} and applying Cauchy--Schwarz with
uniform weights $\Delta t$:
\begin{equation}
  \delta_0^2
  \;\leq\;
  e^{2L}\!\left(\sum_{n=1}^{N}\Delta t\,e_n\right)^{\!2}
  \;\leq\;
  e^{2L}
  \underbrace{\left(\sum_{n=1}^{N}\Delta t\right)}_{=\,1}
  \left(\sum_{n=1}^{N}\Delta t\,e_n^2\right)
  \;=\;
  e^{2L}\sum_{n=1}^{N}\Delta t\,e_n^2,
\end{equation}
which is~\eqref{eq:divergence-bound-full}.
\end{proof}

\begin{remark}[Training-time surrogate]
Proposition~\ref{prop:divergence-bound} bounds action-chunk divergence
by velocity disagreement evaluated at the \emph{nominal Euler trajectory}
points $x^{(0)}_{t_n}$. The cross-view loss $\mathcal{L}_{\mathrm{CV}}$
in~\eqref{eq:paired-cv-loss} evaluates disagreement at the
\emph{training-time linear interpolation} points
$x_{t_k}=t_k\varepsilon+(1-t_k)a$, which differ from
$x^{(0)}_{t_n}$ in general.
The two sets of points coincide when the learned velocity field
produces straight-line trajectories, which holds for the optimal
conditional flow matching solution.
In practice, the training-time objective therefore serves as a
tractable surrogate for the right-hand side
of~\eqref{eq:divergence-bound-full}; this is the sense in which it
is a ``local surrogate'' as described in the main text. The bound is intended as a qualitative motivation rather than a tight quantitative estimate: the constant can be loose for high-capacity neural velocity fields. Its role is to show that controlling velocity disagreement along the integration path is sufficient to control action-chunk divergence, while the training objective provides a practical surrogate for that quantity.
\end{remark}

\begin{lemma}[Shuffled-loss decomposition]
\label{lem:shuffled}
Let $s$ and $s'$ be independent draws from the same state distribution.
Omitting shared arguments $(x_t,t,l,q)$ for readability, the population
shuffled cross-view loss satisfies
\begin{equation}
  \mathbb{E}_{s,s'}\lVert v_p(s')-v_0(s)\rVert^2
  \;=\;
  \mathbb{E}_{s}\lVert v_p(s)-v_0(s)\rVert^2
  +2\,\mathrm{Tr}\!\bigl(\mathrm{Cov}_{s}(v_p(s),v_0(s))\bigr).
  \label{eq:shuffled-decomp-full}
\end{equation}
\end{lemma}

\begin{proof}
Expanding the squared norm and using independence of $s$ and $s'$:
\begin{equation}
  \mathbb{E}_{s,s'}\lVert v_p(s')-v_0(s)\rVert^2
  =\mathbb{E}\lVert v_p\rVert^2
  +\mathbb{E}\lVert v_0\rVert^2
  -2\,\langle\mathbb{E}\,v_p,\,\mathbb{E}\,v_0\rangle,
  \label{eq:shuf-expand}
\end{equation}
where the cross term factorizes as
$\mathbb{E}_{s,s'}[v_p(s')^{\!\top}v_0(s)]
=(\mathbb{E}_{s'}v_p(s'))^{\!\top}(\mathbb{E}_s v_0(s))$
by independence.
The matched loss expands as
\begin{equation}
  \mathbb{E}_{s}\lVert v_p(s)-v_0(s)\rVert^2
  =\mathbb{E}\lVert v_p\rVert^2
  +\mathbb{E}\lVert v_0\rVert^2
  -2\,\mathbb{E}\langle v_p,v_0\rangle.
  \label{eq:match-expand}
\end{equation}
Subtracting~\eqref{eq:match-expand} from~\eqref{eq:shuf-expand}:
\begin{equation}
  \mathbb{E}_{s,s'}\lVert v_p(s')-v_0(s)\rVert^2
  -\mathbb{E}_s\lVert v_p(s)-v_0(s)\rVert^2
  =2\bigl(\mathbb{E}\langle v_p,v_0\rangle
  -\langle\mathbb{E}\,v_p,\mathbb{E}\,v_0\rangle\bigr)
  =2\,\mathrm{Tr}\!\bigl(\mathrm{Cov}_s(v_p,v_0)\bigr),
\end{equation}
which gives~\eqref{eq:shuffled-decomp-full}.
\end{proof}

\section{Implementation and Reproducibility Details}
\label{app:implementation}

\paragraph{Backbone and action expert.}
All policies are $\pi_{0.5}$-style models with the standard pretrained vision-language trunk and a flow-matching action expert~\citep{intelligence2025pi05}. The simulation action space uses the LIBERO 7-DoF action representation. Although the underlying action representation may contain padded dimensions, the cross-view consistency term is applied only to the active action dimensions, $A_{\mathrm{act}}=7$.

\paragraph{Training schedule.}
Unless otherwise stated, simulation policies are trained for $10{,}000$ steps from the same pretrained initialization with AdamW and a cosine learning-rate schedule. Paired runs use batches of $192$ synchronized pairs, $224\times224$ images, and action horizon $H=10$. The primary FM-only same-state-pair control and the proposed configuration use the same paired batches, flow-coordinate sampling, optimization, and training budget; the proposed configuration ramps $\lambda_{\mathrm{CV}}$ from zero to $0.10$ over the first $500$ steps. Both are evaluated over seeds 42, 43, and 44. Final simulation evaluation uses the checkpoint at the end of training. All training configs are provided in the open-sourced code repository linked from the project page.

\paragraph{Image augmentation.}
The final $K=2$ configuration uses no spatial augmentation and applies independent color jitter to the two views. Earlier exploratory single-sample runs used independent spatial augmentation; the corresponding ablation in Table~\ref{tab:app_extended_ablations} is reported separately and is not the recipe used for the main $87.2\%$ camera-track result.

\paragraph{Flow-time and noise sampling.}
For the proposed configuration, we use $K=2$ flow samples per pair. For each pair $i$ and sample $k$, we draw $(t_k^{(i)},\varepsilon_k^{(i)})$ and construct the same noisy action $x_{t_k}^{(i)}$ for both nominal and perturbed views. The same $(t_k^{(i)},\varepsilon_k^{(i)})$ is shared between the two views of a matched pair so that the cross-view term compares velocity predictions at the same point on the action-flow path. The proposed configuration uses $t\sim\mathrm{Beta}(2,3)$, which concentrates flow-time samples in the lower-to-mid range of $t$ (mode at $t=\tfrac{1}{3}$), away from both the pure-noise boundary $t=1$ where velocity predictions are uninformative and the near-action boundary $t=0$ where predictions have already converged.

\paragraph{Bilateral gradient and shuffled-pair control.}
The proposed cross-view loss carries gradients through both the nominal and perturbed branches. The shuffled-pair mode samples a batch-level derangement and applies it only inside the cross-view term; all flow-matching labels remain row-local and correct. This makes the shuffled-pair control a test of action-equivalent pairing rather than a test of corrupted supervised labels.

\paragraph{Compute.}
For the proposed $K=2$ setting, the vision-language trunk forward depends only on the view and can be reused across the $K$ flow samples within a pair; only the action expert is re-evaluated per sample. The per-step wall-clock overhead at $K=2$  ($\approx 2.8 \;s/step$) is therefore below a naive $2\times$ increase over the single-sample variant ($\approx 2.6 \;s/step$). Numbers are measured on a 6$\times$ NVIDIA RTX PRO 6000 Server Edition GPU setup.

\section{Exploratory Feature-Level Alternatives}
\label{app:feature_alternatives}

Before adopting action-flow consistency as the main objective, we explored feature-level alternatives to camera robustness. These experiments are not part of the main evidence for the proposed method, and they are not intended to rule out feature-level methods in general. They explain why the final paper regularizes the policy's own action-flow field rather than an auxiliary representation.

\paragraph{Canonical-token injection.}
One approach injected view-stable canonical tokens (output of a trained transformer module from VGGT hidden layer tensors) into the VLA through an additional cross-attention path. This left the dense visual stream available to the action decoder. The corrected canonical-token injection run reached $15.3\%$ camera-track success, comparable to or below the nominal-only floor. Matched, shuffled, and constant canonical-token controls produced similar behavior, suggesting that the policy did not use the sample-specific canonical content as an action-relevant signal.

\paragraph{Canonical residual action anchor.}
A second approach used canonical tokens to predict a residual action-flow correction gated by a learned scalar. The gate remained nearly closed, and the aggregate camera-track score remained near the nominal-only floor. Forcing the canonical residual branch to be fully active caused nominal-camera rollouts to collapse, indicating that the frozen view-stable representation was not itself an action-executable representation.


\paragraph{Hard action-path bottleneck.}
A further approach inserted a compact bottleneck between the vision-language trunk and the action decoder and masked dense image-token columns from the action suffix. This made the bottleneck action-grounded, but a same-data attribution test did not support a positive method claim: a standard flow-matching policy trained on the same $30\%$ subset reached $71.4\%$, while the regularized bottleneck variant reached $67.4\%$.

\paragraph{Design implication.}
Across these exploratory attempts, a recurring issue was that the policy could use a high-capacity path around the intended view-stable channel. This motivated the final design in the main paper: instead of requiring a particular hidden feature to become invariant, we directly regularize the action-flow predictions produced by the deployed policy. The matched-vs.-shuffled controls in the main paper then test whether this action-level consistency depends on true action-equivalent pairing.

\section{Scope Comparison with Related Camera-Aware Methods}
\label{app:scope_comparison}

\begin{table}[h]
  \centering
  \scriptsize
  \caption{Inference-contract comparison between representative camera-aware VLA methods and the proposed cross-view action-flow consistency. ``Inference-time requirements'' lists inputs or computation needed at deployment beyond what our single scene RGB + language + proprioception interface uses.}
  \label{tab:app_scope_comparison}
  \setlength{\tabcolsep}{4pt}
  \renewcommand{\arraystretch}{1.15}
  \begin{tabular}{p{0.17\linewidth}p{0.21\linewidth}p{0.28\linewidth}p{0.26\linewidth}}
    \toprule
    Family & Representative work & Inference-time requirements & Relation to this work \\
    \midrule
    Camera conditioning
      & \citet{jiang2025cameracond}
      & Camera extrinsics, Pl\"ucker rays, or camera labels supplied to the policy
      & We use no camera labels or extrinsics at inference \\
    Camera-space action
      & \citet{zhang2025ocvla}
      & Extrinsic calibration; actions are defined and predicted in the camera frame
      & We keep actions in the robot-base frame; no calibration is required \\
    Test-time restoration / adaptation
      & \citet{heo2026anycamvla,vlageneralizable2025}
      & A view-restoration or representation-recalibration module applied at deployment
      & No test-time rendering, restoration, or representation-adaptation stage \\
    Training-time view synthesis
      & \citet{tian2024vista}
      & No additional deployed-policy input; synthesized viewpoints are generated before training
      & Like ours, deployment is unchanged, but our objective uses observed action-equivalent views rather than synthesized replacements \\
    3D / geometric inputs
      & \citet{qu2025spatialvla,abouzeid2025geoaware,li2025pointvla,li2025bridgevla,singh2025ogvla}
      & Depth, point clouds, RGB-D, or learned geometric priors as additional inputs
      & A single scene RGB image at inference \\
    Multi-view representation
      & \citet{seo2023mvmwm,pang2025reviwo}
      & A view-invariant representation or multi-view world-model encoder
      & Consistency placed at the action-flow output, not at a hidden representation \\
    \midrule
    \textbf{This work}
      & ---
      & Single scene RGB, language, proprioception; multi-view pairs are used only during training
      & --- \\
    \bottomrule
  \end{tabular}
\end{table}

Scene-camera viewpoint robustness in VLA policies has been approached through several distinct routes with different deployment requirements. Table~\ref{tab:app_scope_comparison} makes those requirements explicit. Camera-conditioning methods require camera labels or extrinsics; camera-space action formulations require extrinsic calibration and a non-base action frame; test-time restoration or representation recalibration adds computation at deployment; 3D/geometric methods require depth, point clouds, RGB-D, or geometric priors. Training-time view synthesis is different: it can leave the deployed policy interface unchanged, but its supervision depends on generated views. The proposed method also leaves deployment unchanged, uses observed action-equivalent views only during training, and places consistency directly on the action-flow velocity field that produces rollout behavior.

The comparison clarifies why the proposed method is positioned as a training-time recipe rather than as an alternative to camera-aware architectures. Methods that change the inference interface are complementary in principle, while training-time view synthesis addresses the same deployment constraint through a different source of viewpoint diversity.

\section{Future Work}
\label{app:future_work}

\paragraph{Cross-VLA-family transfer.}
The proposed objective is stated at the level of an action-flow field $v_\theta$ and can in principle apply to any flow- or diffusion-style VLA action head. We evaluate it only in a $\pi_{0.5}$-style policy. For autoregressive action-token VLAs, the analogous objective would regularize action-token logits or hidden states immediately upstream of the action decoder under the same action-equivalent pairing.

\paragraph{Data-efficient paired fine-tuning.}
The main simulation experiments use a large same-state paired dataset. In practice, a robot user may want to collect a smaller synchronized multi-camera dataset on a target setup and apply cross-view fine-tuning as a short adaptation stage. A useful follow-up is to sweep the number of action-equivalent pairs and fine-tuning steps required to obtain most of the camera-robustness gain.



\end{document}